\documentclass[11pt]{article}
\pdfoutput=1
\usepackage[T1]{fontenc}

\usepackage{microtype}
\usepackage{graphicx}
\usepackage{subfigure}
\usepackage{booktabs} %
\usepackage{amssymb}
\usepackage{bm}
\usepackage{bbm}
\usepackage{amsmath}  %
\usepackage{amsthm}
\usepackage{xcolor}
\usepackage{arydshln}
\usepackage{textcomp}
\usepackage{multirow}
\usepackage{mathtools}
\usepackage[margin=1in]{geometry}
\usepackage{authblk}
\usepackage[numbers,sort&compress]{natbib}
\usepackage{xspace}
\usepackage{comment}
\usepackage{thmtools, thm-restate}  %
\usepackage{verbatim}
\usepackage{arydshln}
\usepackage[normalem]{ulem}
\usepackage{float}
\usepackage{wrapfig2}

\usepackage{tikz}
\usetikzlibrary{arrows}
\usetikzlibrary{positioning}
\usepackage[table]{xcolor}

\tikzset{
  treenode/.style = {align=center, inner sep=0pt, text centered,
    font=\sffamily},
  arn_n/.style = {treenode, circle, black, font=\sffamily\bfseries, draw=black,
    fill=white, text width=1.5em},%
  arn_r/.style = {treenode, circle, black, font=\sffamily\bfseries, draw=black,
    fill=white, text width=1.0em},%
  arn_x/.style = {treenode, rectangle, draw=black,
    minimum width=0.5em, minimum height=0.5em}%
}

\usepackage{hyperref}
\usepackage[capitalize,noabbrev]{cleveref}

\newcommand{\sysname}{\textsc{BeNice!}\xspace}

\usepackage{bm}
\usepackage{enumitem}

\title{BeNice! Closing Looping's Inference-Memory Gap via KV-Sharing}

\author[1,2,*]{Hayden Prairie}
\author[1,*]{Jacob M. Ryan}
\author[1]{Taylor Berg-Kirkpatrick}
\author[1,2]{Daniel Y. Fu}

\affil[1]{University of California, San Diego}
\affil[2]{Together AI}

\date{}

\usepackage{pifont}
\newcommand{\cmark}{\ding{51}}

\newcommand*{\ShowNotes}{}

\newtheorem{theorem}{Theorem}

\ifdefined\ShowNotes
  \newcommand{\colornote}[3]{{\color{#1}\bf{#2 #3}\normalfont}}
\else
  \newcommand{\colornote}[3]{}
\fi

\definecolor{darkred}{rgb}{0.7,0.1,0.1}
\definecolor{darkgreen}{rgb}{0.1,0.5,0.1}
\definecolor{cyan}{rgb}{0.7,0.0,0.7}
\definecolor{dblue}{rgb}{0.2,0.2,0.8}
\definecolor{maroon}{rgb}{0.76,.13,.28}
\definecolor{burntorange}{rgb}{0.81,.33,0}
\definecolor{royalpurple}{rgb}{0.47,.31,0.66}

\ifdefined\ShowNotes
  
\else
  
\fi

\newcommand{\meanrecurrence}{\mu_{\text{rec}}}
\newcommand{\meanbackward}{\mu_{\text{bwd}}}
\newcommand{\pre}{\mathcal{P}}
\newcommand{\rec}{\mathcal{R}}
\newcommand{\OO}{\mathcal{O}}
\newcommand{\coda}{\mathcal{C}}
\newcommand{\dt}{\Delta}
\newcommand{\A}{\bm{A}}
\newcommand{\B}{\bm{B}}
\newcommand{\C}{\bm{C}}

\newcommand{\dA}{\overline{\bm{A}}}
\newcommand{\dB}{\overline{\bm{B}}}

\newcommand{\climbmix}{\texttt{ClimbMix} \cite{diao2025nemotronclimbclusteringbasediterativedata}\xspace}

\newif\ifarxiv

\begin{document}

\maketitle

\renewcommand{\thefootnote}{\fnsymbol{footnote}}
\footnotetext[1]{Equal contribution. Order decided by \href{https://en.wikipedia.org/wiki/Blokus}{Blokus}.}

\vspace{-4.5em}  %
\begin{center}
\texttt{\{hprairie,jmryan,tberg,danfu\}@ucsd.edu}
\end{center}

\begin{abstract}
  
\noindent Looped architectures scale quality while keeping parameter count fixed by repeatedly passing activations through the same block of layers. 
During inference, each recurrence generates independent \emph{KV-caches}, causing memory to grow linearly.
To address this challenge, prior work has proposed KV-sharing (i.e., sharing KV-caches across depth), 
finding that, in some cases, the KV-cache of the final recurrence can be shared with the earlier recurrences
\emph{without quality loss}, a property we denote as \textbf{\emph{KV-shareability}}. 
However, the mechanisms underlying KV-shareability remain poorly understood, with even identically trained looped architectures not consistently exhibiting it.
In this paper, we study how \emph{KV-shareability} can be enabled in looped models.
We propose a two-stage method: in the first stage, we use \textit{stochastic depth} during pretraining, which induces contractive residual dynamics that drive representations toward a common limit across recurrences. In the second stage, we propose BeNice!, a post-training method that broadcasts the final recurrence's KV-cache into the first recurrence, training it to produce compatible residuals while preserving the original model's quality.
As a result, \sysname enables KV-shareability while matching standard finetuning's unshared validation perplexity. 
Using \sysname, we show that KV-sharing is complementary to alternative KV-compression methods, composing with memory-efficient sequence mixers (e.g., multi-latent attention, Kimi delta attention) and low-precision KV-cache quantization (e.g., \texttt{fp8}, \texttt{mxfp8}, \texttt{nvfp4}).
Finally, in simulated low-memory, long-context inference, KV-sharing improves cache hit rates by up to 31.4\%, increases throughput by up to 23.6\%, and reduces p90 time-to-first-token by up to 83.6\%.

\end{abstract}

\section{Introduction}
\label{sec:introduction}

As inference deployments move to the edge \cite{moon2024lpulatencyoptimizedhighlyscalable,narayan2025minionscostefficientcollaborationondevice} and take up a large portion of compute \cite{touvron2023llamaopenefficientfoundation}, looped architectures create a new, predictable avenue to scale quality while keeping parameter count fixed~\cite{geiping_scaling_2025,prairie2026parcaescalinglawsstable,mcleish_retrofitted_recurrence}. 
Unfortunately, while looped architectures keep the memory overhead of parameters capped, they still create \emph{KV-caches at each recurrence} \cite{prairie2026parcaescalinglawsstable,zhuScalingLatentReasoning2025,gao2026looploopies}, 
causing memory overhead to still \emph{grow linearly} with recurrence count at inference time. 
One potential direction to reduce inference-memory overhead is KV-sharing (i.e., sharing KV-caches across recurrences).
Specifically, some recent looped models have been shown to support reuse of the final recurrence's KV-cache for all prior recurrences with \emph{minimal impact to quality} \cite{geiping_scaling_2025,geiping2025efficientparallelsamplersrecurrentdepth}, a property we term \textbf{\emph{KV-shareability}} (\cref{fig:main} [\emph{right}]).

Unfortunately, the mechanisms underlying KV-shareability remain poorly understood, with more recent work reporting significant quality degradation with KV-cache reuse \cite{zhuScalingLatentReasoning2025,gao2026looploopies}.
Moreover, even models trained under identical setups can differ in whether they exhibit KV-shareability (see \cref{tab:kv-sharing}). 
In this paper, we hypothesize that contractive recurrent dynamics drive residual states toward fixed points, whose agreement under shared and unshared KV-caches enables \emph{KV-shareability}.

To encourage contractive residual dynamics, we train with \emph{stochastic depth},  randomly sampling depth throughout training~(\cref{fig:main} [\emph{left}]). 
By requiring low-loss predictions across multiple depths, stochastic depth discourages residual drift and promotes convergence across recurrences. 
Contractive residual dynamics are what allow KV-shareability, which we formally analyze and empirically validate by showing that replacing intermediate caches with the final cache perturbs the attention output by a term bounded by the difference between their input residuals (see \cref{sec:stochastic-fixed}).

Unfortunately, even though \emph{stochastic depth} creates contractive residual dynamics, it does not guarantee that residual states converge to the same fixed point with and without KV-sharing. Instead, KV-sharing sometimes introduces token-independent residual errors in the first recurrence, shifting the model's fixed point and harming downstream quality (see \cref{sec:compression-failures}).

To enable KV-shareability, we pair stochastic depth pretraining with a targeted correction of the residual errors introduced by KV-sharing. We propose a two-stage method: first, \emph{stochastic depth} pretraining to induce contractive residual dynamics; second, \sysname to correct first recurrence residual errors created from KV-sharing. 
Specifically, \sysname jointly optimizes both the first recurrence with KV-cache reuse and the final recurrence without reuse (\cref{fig:main} [\emph{middle}]), balancing KV-shareability and base-model quality.
To further study KV-sharing with other KV-compression techniques, we extended \sysname to memory-efficient sequence mixers, including multi-latent attention (MLA) \cite{deepseekai2024deepseekv2strongeconomicalefficient} and Kimi delta attention (KDA) \cite{kimiteam2025kimilinearexpressiveefficient}, by reusing the final recurrence's latent or state, respectively.
Overall, \sysname unlocks near-lossless KV-cache reuse across recurrences, facilitating predictable quality scaling using extended looping \cite{prairie2026parcaescalinglawsstable} while keeping parameter count and KV-cache memory constant as depth grows. 

    \begin{figure}[!t]
        \centering
        \vspace{-1em}
       \includegraphics[width=\linewidth]{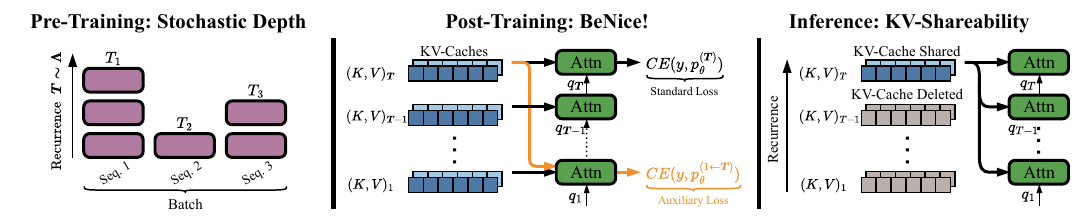}
        \vspace{-2em}
        \caption{\textbf{Stochastic Depth, \sysname, and KV-Shareability.} (\emph{Left}) During pretraining, stochastic depth induces contractive residual dynamics. (\emph{Middle}) \sysname post-trains by injecting the final recurrence KV-caches into the first recurrence. Together, they enable KV-shareability in looped models, retaining only the final recurrence's KV-cache \emph{without degrading quality} (\emph{Right}).}
        \label{fig:main}
        \vspace{-1em}
    \end{figure}

We evaluate KV-shareability and \sysname on end-to-end quality, compatibility with other KV compression methods, and how it accelerates inference workloads:

\begin{itemize}[itemsep=0pt,topsep=0pt,leftmargin=*]
    \item \textbf{End-to-End Quality.} \sysname improves shared quality by up to 8.35 and 4.56 points for Core and Core-extended \cite{li2025datacomplmsearchgenerationtraining}, respectively. Compared with standard finetuning, \sysname improves validation perplexity by 9.75 points when sharing, while matching quality in the unshared case.
    \item \textbf{Complementary to other KV Compression Methods}. To evaluate \emph{KV-sharing} with other KV-compression methods, we study its composition with memory-efficient sequence mixers and KV-quantization. We find that KV-sharing composes with both MLA and KDA sequence mixers and low-precision KV-quantization (e.g., \texttt{fp8}, \texttt{mxfp4}, \texttt{nvfp4}).
    \item \textbf{Accelerating Inference Workloads.} Under simulated long-context inference traffic, we observe that memory reductions in KV-sharing improve cache-hit rates by up to 31.4\% (8.1 points), increasing throughput by up to 23.6\%, and reducing p90 TTFT by up to 83.6\%.
\end{itemize}

\section{Background}
\label{sec:background}

We first provide a brief background on prior looped models (\cref{sec:background:looped}) and the inference overhead of language models (\cref{sec:background:inference-overhead}). 
Looped architectures have been widely explored in synthetic problem solving tasks \cite{bansalEndtoendAlgorithmSynthesis2022,jolicoeur-martineauLessMoreRecursive2025,avi_learn_algorithm, wangHierarchicalReasoningModel2025b,huang2026equilibriumreasonerslearningattractors}; however, our work focuses on looped language models \cite{dehghaniUniversalTransformers2019,zhuScalingLatentReasoning2025,geiping_scaling_2025,prairie2026parcaescalinglawsstable}. These models differ in where recurrence occur (pre-, mid, post-, or full-looping) \cite{alabdulmohsin2025recursiveinferencescalingwinning}, how it is halted (i.e., explicit router \cite{zhuScalingLatentReasoning2025} or implicitly within the residual \cite{geiping_scaling_2025,prairie2026parcaescalinglawsstable}), and how gradients are computed (explicit or implicit backpropagation \cite{bai2019deepequilibriummodels}). We study middle-looped models using implicit halting and explicit differentiation, with an emphasis on how pretraining for fixed-point convergence enables memory-efficient inference.
An extended literature review is found in \cref{sec:extended-lit}.

\subsection{Prior Looped Models}
\label{sec:background:looped}

Most prior looped language models \cite{prairie2026parcaescalinglawsstable,geiping_scaling_2025,zhuScalingLatentReasoning2025} follow a middle-looped structure sequentially composed of an initial \emph{prelude} block $\pre$, a middle \emph{recurrent} block $\rec$, and a final \emph{coda} block $\coda$. For an input sequence $s \in \mathcal{V}^n$, computation proceeds sequentially as $e = \pre(s),~h_{t+1}=\rec(h_t, e),~p=\coda(h_T)$, where a model is looped $T$ times and $h_0$ is initialized randomly (i.e., $h_0\sim\mathcal{N}(0, I\sigma^2)$;~\cite{geiping_scaling_2025,prairie2026parcaescalinglawsstable}) or with $e$ (i.e., $h_0 \gets e$;~\cite{zhuScalingLatentReasoning2025}). Intuitively, $\pre$ maps a sequence $s$ into a latent embedding $e\in\mathbb{R}^{n \times d_e}$, $\rec$ iteratively refines a hidden state $h_t \in \mathbb{R}^{n \times d_h}$ conditioned on $e$, and then $\coda$ maps the final $h_T$ to an output $p \in \mathbb{R}^{n \times |\mathcal{V}|}$. To condition $\rec$ on $e$, input injection at each recurrence is often used (i.e., $h_{t+1} = \rec(\dA \cdot h_t + \dB e)$, where $\dA \in \mathbb{R}^{d_h \times d_h}$ and $\dB \in \mathbb{R}^{d_h \times d_e}$; \cite{prairie2026parcaescalinglawsstable,yang2024loopedtransformersbetterlearning,geiping_scaling_2025}), and following \citet{prairie2026parcaescalinglawsstable} is constrains $\rho(\dA) < 1$ where $\rho(\cdot)$ is the spectral radius, to prevent state explosion.

During training, prior works either sample depth stochastically \citep{geiping_scaling_2025,prairie2026parcaescalinglawsstable} or use adaptive computation time (ACT) \cite{graves2017adaptivecomputationtimerecurrent}. Specifically, Parcae \cite{prairie2026parcaescalinglawsstable} and Huginn \cite{geiping_scaling_2025} sample a depth $T\sim\Lambda$ independently per sequence/batch, where $\Lambda$ is a distribution (e.g., Poisson) with mean $\meanrecurrence$, training $\rec$ to produce meaningful representations across depths. Ouro \cite{zhuScalingLatentReasoning2025} uses ACT: a learned router predicts a halting probability at each recurrence, with computation stopping once the cumulative probability exceeds a threshold \cite{graves2017adaptivecomputationtimerecurrent}. We will show the advantages of training with stochastically sampled depth in this work. Finally, to bound training memory, some models use truncated backpropagation through time (TBPTT) \cite{geiping_scaling_2025,prairie2026parcaescalinglawsstable}, propagating gradients through only the final $\meanbackward < T$ recurrences.

\subsection{Inference Overhead of Looped Language Models}
\label{sec:background:inference-overhead}

We will study how \emph{memory overhead at inference} scales with context length and model depth, particularly within the scope of looped architectures \cite{geiping_scaling_2025,prairie2026parcaescalinglawsstable}. Under uniform numerical precision, inference memory $M$ for a $P$-parameter model, with $L$ context length, $D$ depth, and $d_\text{kv}$ KV width is $M=M_\theta + M_{\text{kv}}$, where model-weight memory is $M_\theta=\OO(P)$ and KV-cache memory is $M_\text{kv} = \OO(LDd_{kv})$.
Beyond lowering numerical precision, prior work reduces $M_\text{kv}$ overhead primarily by shrinking per-layer footprint $Ld_\text{kv}$ (e.g., MLA \cite{deepseekai2024deepseekv2strongeconomicalefficient}, GQA \cite{ainslie2023gqatraininggeneralizedmultiquery}, KDA \cite{kimiteam2025kimilinearexpressiveefficient}), leaving its linear dependence on depth $D$ unchanged.
For a looped model evaluated with $T$ recurrences, the effective depth is $D = D_\pre + T \cdot D_\rec + D_\coda$. Although looping keeps $M_\theta$ fixed as $T$ increases, conventional inference stores a distinct KV cache for each recurrence, causing $M_\text{kv}$ to grow linearly with $T$.

\begin{table}[!t]
    \centering
    \setlength{\tabcolsep}{4pt} %
    \renewcommand{\arraystretch}{1.15} %
    \small
    \begin{tabular}{c|l|ccc|ccc|ccc}
        \toprule
        & \multirow{2}{*}{\textbf{Model}} & \multicolumn{3}{c}{\textbf{Lambada BPB} ($\downarrow$)} & \multicolumn{3}{c}{\textbf{Core} ($\uparrow$)} & \multicolumn{3}{c}{\textbf{Core-Extended} ($\uparrow$)} \\
        \cmidrule(lr){3-5} \cmidrule(lr){6-8} \cmidrule(lr){9-11}
        & & Unshared & Shared & $\mathbf{\Delta}$ & Unshared & Shared & $\mathbf{\Delta}$  & Unshared & Shared & $\mathbf{\Delta}$  \\
        \midrule
        \multirow{5}{*}{\rotatebox[origin=c]{90}{\emph{Fixed-Point}}}
& \cellcolor{gray!10}\textbf{Parcae 140M}
& \cellcolor{gray!10}0.962
& \cellcolor{gray!10}0.962
& \cellcolor{gray!10}\textbf{0.000}
& \cellcolor{gray!10}14.07
& \cellcolor{gray!10}14.17
& \cellcolor{gray!10}\textbf{0.10}
& \cellcolor{gray!10}10.13
& \cellcolor{gray!10}10.09
& \cellcolor{gray!10}\textbf{-0.04} \\

        & \cellcolor{gray!10}\textbf{Parcae 370M}    & \cellcolor{gray!10}0.762 & \cellcolor{gray!10}0.764 & \cellcolor{gray!10}\textbf{0.002} & \cellcolor{gray!10}20.47 & \cellcolor{gray!10}19.85 & \cellcolor{gray!10}\textbf{-0.62} & \cellcolor{gray!10}13.23 & \cellcolor{gray!10}13.24 & \cellcolor{gray!10}\textbf{0.01} \\
        & Parcae 770M    & 0.645 & 1.460 & 0.814 & 25.05 & 21.10 & -3.94 & 15.69 & 13.95 & -1.74 \\
        & Parcae 1.3B    & 0.601 & 2.281 & 1.680 & 28.76 & 21.37 & -7.40 & 16.73 & 13.48 & -3.25 \\
                & \cellcolor{gray!10}\textbf{Huginn 3.5B}    & \cellcolor{gray!10}0.429 & \cellcolor{gray!10}0.429 & \cellcolor{gray!10}\textbf{0.000} & \cellcolor{gray!10}40.97 & \cellcolor{gray!10}40.53 & \cellcolor{gray!10}\textbf{-0.44} & \cellcolor{gray!10}29.06 & \cellcolor{gray!10}28.61 & \cellcolor{gray!10}\textbf{-0.44} \\
                \cdashline{1-11}
        \multirow{2}{*}{\rotatebox[origin=c]{90}{\emph{ACT}}}
        & Ouro 1.4B      & 0.353 & 1.280 & 0.927 & 54.57 & 25.17 & 29.40 & 53.11 & 32.13 & 20.98 \\
        & Ouro 2.6B      & 0.301 & 0.744 & 0.443 & 59.78 & 26.58 & 33.20 & 61.42 & 37.57 & 23.84 \\
        \bottomrule
    \end{tabular}
    \caption{\textbf{Emergence of KV-shareability in Looped Models.} Quality of looped architectures with and without a shared KV-cache. Observe that KV-shareability is \emph{not universal} to all looped architectures (\colorbox{gray!10}{\textbf{Bolded} models have KV-shareability}). See \cref{sec:eval-setup} for evaluation details.}
    \label{tab:kv-sharing}
\end{table}

\section{Understanding KV-shareability in Looped Models}
\label{sec:elucdiation}

In this section, we study \emph{KV-shareability} in looped architectures.
To evaluate KV-shareability, we compare \emph{shared evaluations}, which reuse the final-recurrence cache across all recurrences, against \emph{unshared evaluations}, which retain every recurrence's cache.
Details on reproducing accurate KV-sharing inference setups for our benchmarks are provided in \cref{sec:eval-setup}.

We motivate our investigation by observing that KV-shareability is \emph{not universal to all looped models} (see \cref{tab:kv-sharing}). We establish that \emph{stochastic depth} enables contractive recurrent dynamics, which we formally analyze and empirically validate allows KV-shareability when fixed points agree under shared and unshared KV-caches (\cref{sec:stochastic-fixed}). However, stochastic depth alone is insufficient to guarantee agreement, leading to models not exhibiting KV-shareability (e.g., Parcae 1.3B in \cref{tab:kv-sharing}).
Analyzing failure cases, we find that KV-sharing creates a token-independent basis shift, causing convergence to different fixed points and degraded quality under KV-sharing (\cref{sec:compression-failures}).

\subsection{Stochastic Depth Enables Contractive Residual Dynamics}
\label{sec:stochastic-fixed}

In this section, we characterize how \emph{stochastic-depth} training induces fixed-point convergence, driving KV-caches near a common limit and correlating with the emergence of KV-shareability. 

\paragraph{Setup.} We train Parcae-style \cite{prairie2026parcaescalinglawsstable} models from scratch at four scales (140M, 370M, 770M, and 1.3B), varying only whether recurrence depth is stochastic (i.e., randomly sampled during training) or static (i.e., fixed throughout training).  KV-shareability is evaluated on final models and never trained for explicitly. Training and evaluation details are in \cref{sec:expanded-stochastic-fixed} and \cref{sec:eval-setup}, respectively.

\begin{figure}[!t]
    \centering
    \vspace{-1em}
    \includegraphics[width=\linewidth]{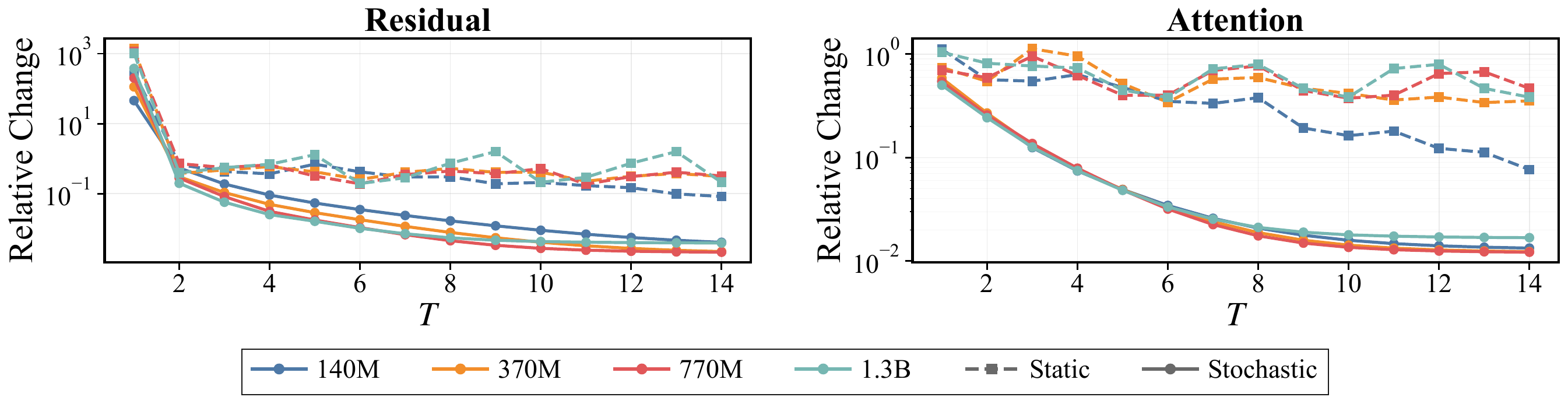}
    \vspace{-1.7em}
    \caption{\textbf{Residuals and Attention Convergence.} Relative step-to-step change of the residual stream ($\lVert h_i - h_{i-1}\rVert / \lVert h_{i-1}\rVert$, \emph{left}) and attention pattern ($\lVert A_i - A_{i-1}\rVert / \lVert A_{i-1}\rVert$, \emph{right}) across recurrence depth $T$. \emph{Stochastic-depth} training exhibits fixed-point convergence, while \emph{static-depth} does not.}
    \label{fig:attn_residual_convergence}
\end{figure}

\begin{table}[!t]
    \centering
    \begin{minipage}[t]{0.5\textwidth}
        \vspace{0pt}
        \centering
        \small
        \begin{tabular}{l|cccc}
            \toprule
            \multirow{2}{*}{\textbf{Size}} & \multicolumn{2}{c}{$\mathbf{\Delta}$ \textbf{Val PPL} ($\downarrow$)} & \multicolumn{2}{c}{$\mathbf{\Delta}$\textbf{Core} ($\downarrow$)} \\
            \cmidrule(lr){2-3} \cmidrule(lr){4-5}
            & \emph{Stochastic} & \emph{Static} & \emph{Stochastic} & \emph{Static} \\
            \midrule
            140M & \textbf{-0.02} & 1.31 & \textbf{0.08} & 0.62 \\
            370M & \textbf{0.02} & 2.70 & \textbf{0.13} & 0.78 \\
            770M & \textbf{0.01} & 3.10 & \textbf{0.13} & 1.06 \\
            1.3B & \textbf{0.01} & 2.84 & \textbf{0.25} &  2.00 \\
            \bottomrule
        \end{tabular}
        \caption{\textbf{KV-Sharing Relies on Stochastic Depth.} We report the quality delta between unshared and shared KV-caching. Extended results and training details are in \cref{sec:expanded-stochastic-fixed}.}
        \label{tab:emergence-coincides-stochastic-depth}
    \end{minipage}
    \hfill
    \begin{minipage}[t]{0.48\textwidth}
        \vspace{-0.3em}
        \centering
        \makeatletter\def\@captype{figure}\makeatother
        \includegraphics[width=\linewidth]{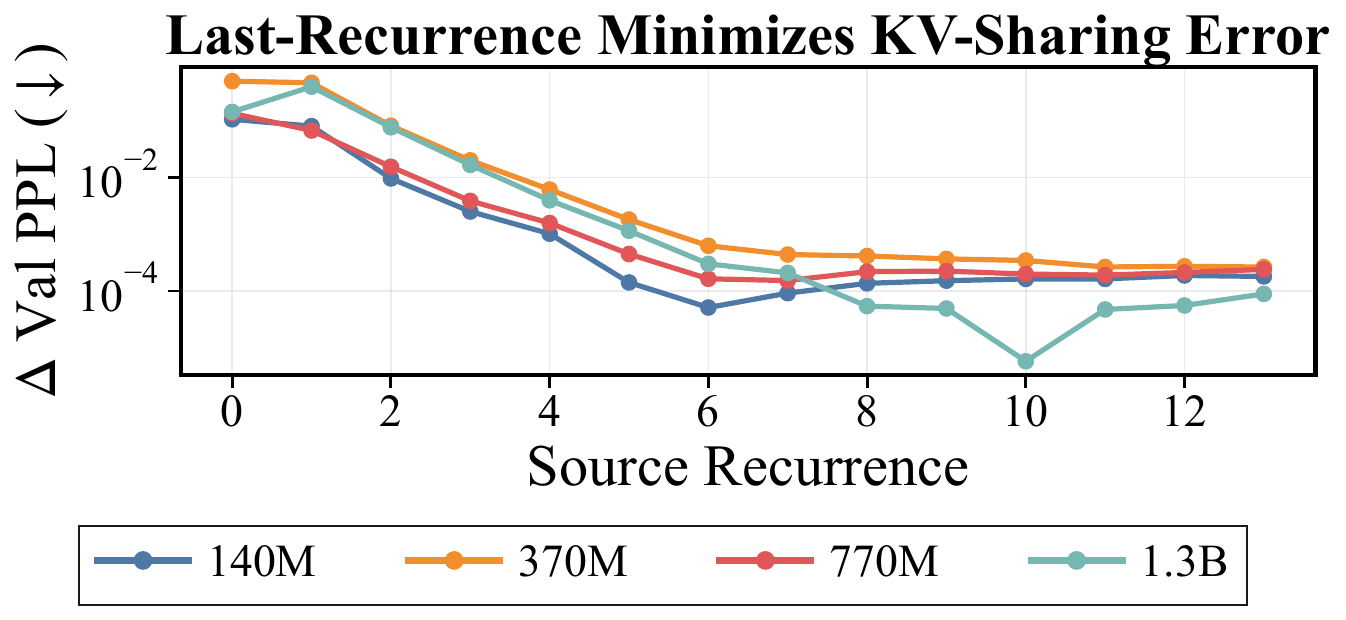}
        \vspace{-1.7em}
        \caption{\textbf{Final-Recurrence KV is Optimal.} We find the best quality results are from converged KV caches. }
        \label{fig:final_recurrence_optimal}
    \end{minipage}
    \vspace{-1em}
\end{table}

\paragraph{Fixed-Point Convergence encourages Shareable KV-Caches.}
Our key insight is that \emph{KV-shareability} often emerges as a downstream consequence of fixed-point convergence in the recurrent block $\rec$.
Specifically, consider a looped architecture in which $\rec$ is a contraction on the hidden state $h_t$, such that $h_t \to \hat h$ where $\hat h$ is a fixed point (e.g., Parcae \cite{prairie2026parcaescalinglawsstable}; Huginn \cite{geiping_scaling_2025}; see \cref{sec:parcae-is-fixed-point} for empirical evidence).
Since the key and value projections are locally-Lipschitz in $h$, their caches also converge: $k_t^{(l)},v_t^{(l)} \to \hat k^{(l)}, \hat v^{(l)}$ for every layer $l$ in $\rec$. Moreover, $\lVert k_t^{(l)} - k_T^{(l)}\rVert$ is bounded by the step-to-step residual $\lVert h_t - h_{t-1}\rVert$ (see \cref{sec:convergence-of-attention}), as illustrated in \cref{fig:attn_residual_convergence}.

KV-cache convergence is necessary for KV-shareability: replacing an intermediate cache $(k_t^{(l)}, v_t^{(l)})$ with a fixed-point cache $(\hat k^{(l)}, \hat v^{(l)})$ perturbs the attention output by a term bounded by the fixed-point residual $\lVert \hat h - h_t\rVert$. To test this, we train Parcae models from scratch using stochastic and static depth training. Stochastic depth training induces a contraction\footnote{Intuitively, stochastic-depth training induces a contraction because a single set of tied weights must produce low-loss latents at many depths under a shared readout, discouraging successive states from drifting.} and coincides with the emergence of KV-shareability (\cref{tab:emergence-coincides-stochastic-depth}). In contrast, static-depth lacks an incentive to keep intermediate latents near a common limit, as reflected in their attention patterns (\cref{fig:attn_residual_convergence}) and shared quality (\cref{tab:emergence-coincides-stochastic-depth}).

Consistent with this mechanism, the final-recurrence cache is empirically the optimal broadcast target (\cref{fig:final_recurrence_optimal}), as $(k_T^{(l)}, v_T^{(l)})$ lie closest to a fixed point across all $t$, minimizing substitution error. 
While KV-shareability is still surprising, it suggests that looped architectures should favor stochastic fixed-point training over ACT objectives. By inducing a contraction towards a common fixed point, stochastic-depth training can enable KV-shareability without explicitly optimizing for it.

\begin{figure}
    \centering
    \includegraphics[width=\linewidth]{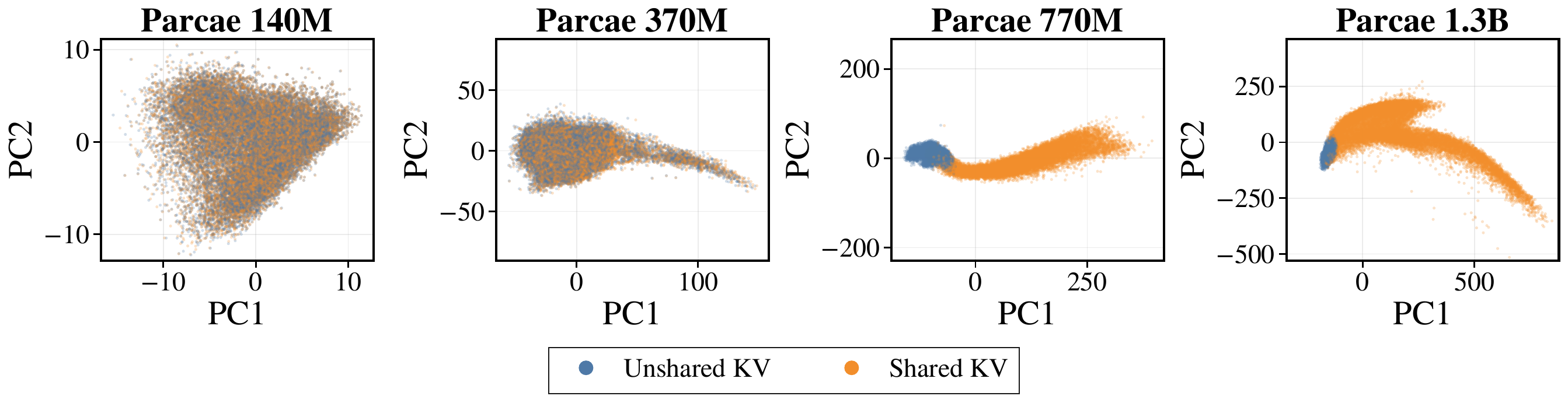}
    \vspace{-1.7em}
    \caption{\textbf{KV-Sharing Failures Converge to Incorrect Fixed Points.} We visualize fixed points of several tokens after $T=8$ test-time recurrence with unshared and shared KV-caches. Models with KV-shareability (e.g., 140m and 370m) converge to the same fixed point under both setups, while models without KV-shareability (e.g., 770M and 1.3B) do not. See \cref{sec:fixed-point-no-good} for extended details. }
    \label{fig:fixed-point-failure}
\end{figure}

\begin{table}[!t]
    \centering
    \begin{minipage}[t]{0.48\textwidth}
        \setlength{\tabcolsep}{6pt}
        \renewcommand{\arraystretch}{1.05}
        \centering
        \begin{tabular}{@{}l|cc@{}}
            \toprule
    \textbf{Cache Compression}
        & \multicolumn{2}{c}{\textbf{Val. PPL.} $(\downarrow)$} \\
    \cmidrule(l){2-3}
        & \textbf{770M}
        & \textbf{1.3B} \\
            \midrule
            \emph{Unshared}  & 11.86 & 10.80 \\
                \cdashline{1-3}
            \emph{Layer $> 1$, Recurrence  $> 1$ } & 11.86 &  10.80 \\
            \emph{Naive Patching} &  11.87 & 10.84 \\ 
                \cdashline{1-3}
            \emph{Shared} & 34.24 & $1.71e4$ \\
            \bottomrule
        \end{tabular}
        \caption{\textbf{Pinpointing Parcae Failures.} Preserving first-layer KVs from the first recurrence or applying naive patching matches the unshared baseline, whereas full KV-sharing fails.
        }
        \label{tab:compression-patching}
    \end{minipage}
    \hfill
    \begin{minipage}[t]{0.48\textwidth}
        \vspace{-5.2em}
        \centering
        \makeatletter\def\@captype{figure}\makeatother
         \includegraphics[width=\linewidth]{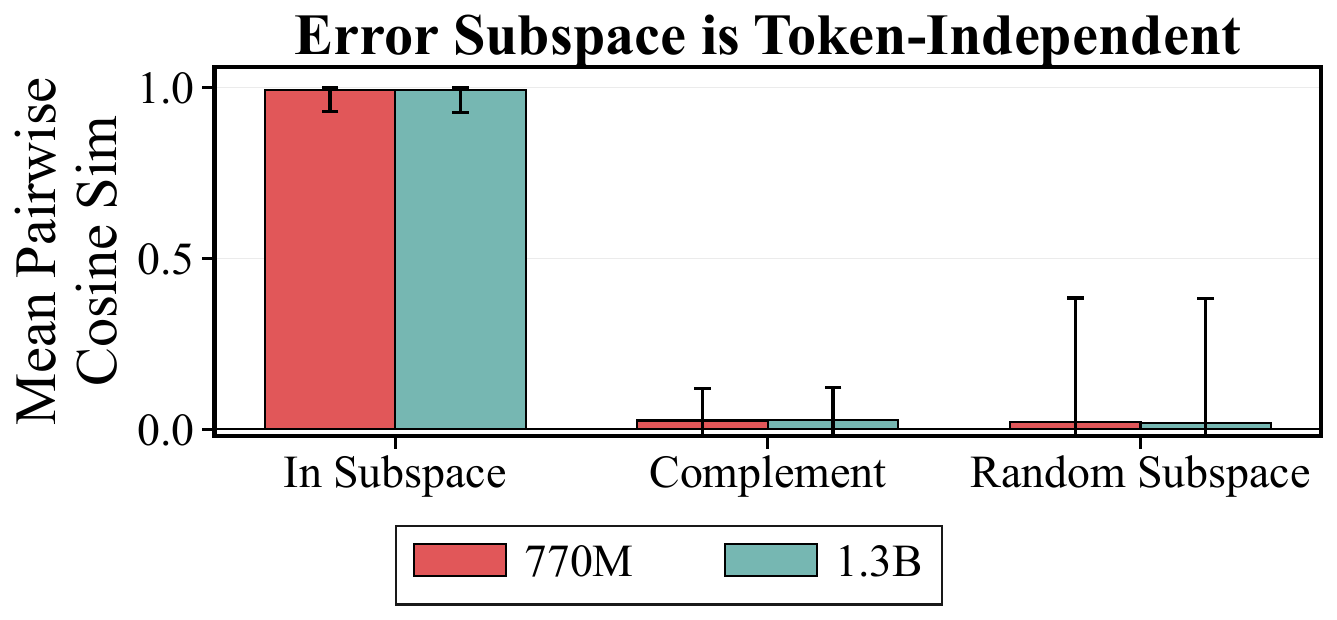}
        \vspace{-1.76em}
        \caption{\textbf{KV-Sharing Error Subspace Similarity.} We find that tokens are strongly aligned within the subspace, but orthogonal in its complement and a random 8-dimensional subspace.}
    \label{fig:mean-pairwise-cosine}
    \end{minipage}
    \vspace{-1em}
\end{table}

\subsection{KV-Sharing Fails to Converge to a Common Fixed Point}
\label{sec:compression-failures}

We now investigate why \emph{contractive residual dynamics} alone are insufficient (e.g., 770M and 1.3B Parcae use stochastic depth and are fixed-point [\cref{sec:parcae-is-fixed-point}], yet lack KV-shareability [\cref{tab:kv-sharing}]).

\paragraph{Setup.}
We focus our analysis on the released Parcae models \cite{prairie2026parcaescalinglawsstable}, focusing on 770M and 1.3B scales that both fail to preserve quality under KV-sharing, yet were trained with stochastic depth and an equivalent setup to 140M and 370M Parcae models. We evaluate them using the setup in \cref{sec:eval-setup}.

\paragraph{KV-Sharing Fails to Converge to Similar Fixed Points.} \cref{fig:fixed-point-failure} visualizes the converged representation with and without KV-sharing.
We quantify the fixed-point agreement by the mean relative difference between final latents, $\|h_T^{(\text{shared})}-h_T^{(\text{unshared})}\|/\|h_T^{(\text{shared})}\|$.
Models with \emph{KV-shareability}  (e.g., Parcae 140M and 370M in \cref{tab:kv-sharing}) converge to similar fixed points, with a mean relative difference of at most 1.69\%. In contrast, models without \emph{KV-shareability} (e.g., Parcae 770M and 1.3B in \cref{tab:kv-sharing}) exhibit mean relative differences reaching up to 177.5\%.

\paragraph{KV-Sharing Failures are Token-Independent Basis Shifts}
Analyzing the source of KV-sharing failure, we observe that caching only the first layer KV of the first recurrence, on top of all final recurrence caches, restores quality (i.e., \emph{Layer $> 1$, Recurrence $>1$} in \cref{tab:compression-patching}). 
This localizes the failure to just the first layer of the first recurrence.
More surprisingly, this failure manifests as a token-independent basis shift, altering the fixed point under KV-sharing and degrading quality. 

\cref{tab:compression-patching} and \cref{fig:mean-pairwise-cosine} confirm this empirically: by constructing a subspace $B$ from the error delta between the first layer output with and without KV-sharing (see \cref{sec:low-dim-construction} for derivation) and projecting each token into $B$, we observe an average projected pairwise cosine similarity of 0.99. 
This confirms that KV-sharing errors act as a token-independent residual basis shift.
To test this, we add the mean projected error to the residual after the first layer during shared evaluations (see \cref{sec:low-dim-construction} for details).
This correction eliminates the need to separately cache the first layer and recurrence when sharing (see \emph{Naive Patching} in \cref{tab:compression-patching}).

\section{Learning to Share Across Recurrences}
\label{sec:methodology}

One way to induce KV-shareability is to train directly on KV-sharing rollouts (i.e., sequentially rolling out KV-sharing dynamics token-by-token). However, this would be computationally infeasible and slow during training. 
Alternatively, using our analysis in \cref{sec:elucdiation}, we propose a two-stage method to enable \emph{KV-shareability} (as visualized in \cref{fig:main}). In the first stage, we follow our observations in \cref{sec:stochastic-fixed} and pretrain models with stochastic depth to induce contractive residual dynamics that drive KV-caches to a common limit across recurrences (\cref{sec:methodology:stochastic}). In the second stage, we correct the token-independent basis shift observed in \cref{sec:compression-failures} with \sysname by optimizing both the first recurrence with KV-cache reuse and the standard final recurrence without reuse (\cref{sec:methodology:posttraining}).

\subsection{Stochastic-Depth Pretraining: Enabling Contractive Recurrent Dynamics}
\label{sec:methodology:stochastic}

In the first stage of our method, we train with \emph{stochastic depth}, driving KV-caches to a common shareable limit across recurrences and setting the groundwork for \emph{KV-shareability} based on \cref{sec:stochastic-fixed}. 
Specifically, following \citet{prairie2026parcaescalinglawsstable}, we utilize per-sequence stochastic depth training, assigning a different depth for each sequence in a batch. This sampling structure was empirically observed to reduce loss spikes in prior work.
We further follow \citet{prairie2026parcaescalinglawsstable}, sampling per-sequence depth from $T\sim \Lambda(\meanrecurrence)$, where $\Lambda$ is a Poisson distribution with mean $\meanrecurrence$.
Lastly, to keep memory bounded during pretraining, we use truncated backpropagation through recurrences, with a backprop depth of $\meanbackward = \lfloor \frac{\meanrecurrence}{2} \rfloor$ \cite{prairie2026parcaescalinglawsstable}.

\subsection{\sysname Post-training: Correcting First Recurrence Residual Errors}
\label{sec:methodology:posttraining}

In the second stage of our method, we correct token-independent issues that still emerge when KV-sharing, even with stochastic-depth training. Using our analysis from \cref{sec:compression-failures}, we know quality degradations are primarily introduced in the first recurrence (see \cref{tab:compression-patching}).
This suggests that KV-sharing compatible inference can be recovered by training the first recurrence pass to be compatible with future KV-caches, rather than directly creating KV-sharing rollouts.
Specifically, consider the ordinary $T$-recurrence predictor $p_\theta^{(T)}$ and a hybrid predictor $p_\theta^{(1\leftarrow T)}$ that uses first-recurrence queries and residuals with final-recurrence KV (see \cref{fig:main} [\emph{middle}]).
We define the \sysname loss as:
\begin{equation}
\mathcal{L}_{\mathrm{\sysname}} = \operatorname{CE}\!\left(y,p_\theta^{(T)}\right) + \gamma\operatorname{CE}\!\left(y,p_\theta^{(1\leftarrow T)}\right)
\label{eq:swap-ce}
\end{equation}
where $\mathrm{CE}(\cdot)$ denotes cross-entropy loss and $\gamma$ controls auxiliary loss strength.
Intuitively, the auxiliary term trains the first recurrence to be compatible with future recurrences' KV-caches, while the standard language-modeling term maintains base quality at later depths.
Finally, during finetuning, we sample $T$ from the same recurrence-depth distribution $\Lambda$ used during pretraining.

Additionally, to study if KV-sharing is complementary to other KV compression methods, we further generalize \sysname to multiple memory-efficient sequence mixers. For example, we can extend \sysname to multi-latent attention \cite{deepseekai2024deepseekv2strongeconomicalefficient} and Kimi delta attention \cite{kimiteam2025kimilinearexpressiveefficient} by swapping the final recurrence's latent or hidden state, respectively, over the KV-cache (see \cref{sec:extended-mla-kda} for details).

\section{Results}
\label{sec:results}

We evaluate end-to-end quality of finetuning methods for enabling KV-shareability (\cref{sec:quality}), its composition with other KV compression techniques (\cref{sec:reduce-costs}), and its latency and throughput gains in memory-bound setups (\cref{sec:inference-workloads}).
Overall, we find \sysname enables KV-shareability, improving shared Core and Core-Extended \cite{li2025datacomplmsearchgenerationtraining} quality by up to 8.35 and 4.56 points, respectively. 
Furthermore, \sysname is complementary to memory-efficient sequence mixers (e.g., MLA \cite{deepseekai2024deepseekv2strongeconomicalefficient}, KDA \cite{kimiteam2025kimilinearexpressiveefficient}) and KV-quantization.
Lastly, in synthetic memory-constrained inference workloads, KV-shareability improves cache-hit rate by up to 31.4\% (8.1 percentage points), increases throughput by up to 23.6\%, and reduces p90 TTFT by up to 83.6\%.

\subsection{Enabling KV-Shareability while Maintaining End-to-End Quality}
\label{sec:quality}

On newly trained Parcae~\cite{prairie2026parcaescalinglawsstable}, \sysname enables KV-shareability and improves base quality. \sysname achieves comparable results to standard finetuning, while enabling KV-shareability.

\begin{table}[!t]
    \centering
    \setlength{\tabcolsep}{6pt} %
    \renewcommand{\arraystretch}{1.05} %
    \small
\begin{tabular*}{\textwidth}{
        @{}cl|@{\hspace{0em}\extracolsep{\fill}}ccc ccc@{}
    }
        \toprule
        & \multirow{2}{*}{\textbf{Model}} & \multicolumn{3}{c}{\textbf{Core} ($\uparrow$)} & \multicolumn{3}{c}{\textbf{Core-Extended} ($\uparrow$)} \\
        \cmidrule(lr){3-5} \cmidrule(lr){6-8}
        & & Unshared & Shared & $\mathbf{\Delta}$ & Unshared & Shared & $\mathbf{\Delta}$  \\
        \midrule
        \multirow{2}{*}{\rotatebox[origin=c]{90}{\small 140M}} 
& Parcae & \textbf{15.50} & 14.97 & -0.53 & \textbf{10.43} & 10.10 & -0.33 \\
& \textbf{+ \sysname} & 15.47 & \textbf{15.61} & \textbf{+0.14} & 10.28 & \textbf{10.19} & \textbf{-0.09} \\
            \midrule
        \multirow{2}{*}{\rotatebox[origin=c]{90}{\small 370M}} 
& Parcae & 20.74 & 15.16 & -5.58 & 12.99 & 9.68 & -3.31 \\
& \textbf{+ \sysname} & \textbf{23.90} & \textbf{23.51} & \textbf{-0.39} & \textbf{14.39} & \textbf{14.24} & \textbf{-0.15} \\
            \midrule
        \multirow{2}{*}{\rotatebox[origin=c]{90}{\small 770M}}
& Parcae & 30.45 & 26.92 & -3.52 & \textbf{19.22} & 16.67 & -2.55 \\
& \textbf{+ \sysname} & \textbf{30.79} & \textbf{30.53} & \textbf{-0.26} & 19.18 & \textbf{19.13} & \textbf{-0.05} \\
            \midrule
        \multirow{2}{*}{\rotatebox[origin=c]{90}{\small 1.3B}} 
& Parcae & 37.27 & 33.94 & -3.33 & 25.41 & 23.46 & -1.95 \\
& \textbf{+ \sysname} & \textbf{37.40} & \textbf{37.43} & \textbf{+0.04} & \textbf{25.69} & \textbf{25.59} & \textbf{-0.10} \\
        \bottomrule
    \end{tabular*}
    \caption{\textbf{\sysname Maintains Quality and Enables KV-Shareability.} We train Parcae models \cite{prairie2026parcaescalinglawsstable} and then post-tune with \sysname. \sysname maintains and improves base model quality, while fully enabling KV-sharing. \textbf{Best} results are \textbf{bolded}.}
    \label{tab:quality}
    \vspace{-.7em}
\end{table}

\paragraph{Setup.}
We follow Parcae \cite{prairie2026parcaescalinglawsstable}, training 140M, 370M, 770M, and 1.3B model classes (see \cref{sec:model-configurations} for definitions) on over 100B \texttt{climbix} \cite{diao2025nemotronclimbclusteringbasediterativedata} tokens.
Baselines~\cite{prairie2026parcaescalinglawsstable} are tuned by sweeping hyperparameters, using our hyperparameter-transfer observations in \cref{sec:hyperparameter-transfer}. 
We finetune \sysname with three seeds, and then evaluate using a test-time recurrence of $T=\meanrecurrence$, following \citet{prairie2026parcaescalinglawsstable}.
Individual hyperparameter sweeps are done for standard finetuning and \sysname (see \cref{sec:hyperparameters}).

\paragraph{\sysname Enables KV-Shareability and Maintains Base Model Quality.} 
\cref{tab:quality} shows that KV-sharing failures emerge at all scales, reducing Core and Core-Ext \cite{li2025datacomplmsearchgenerationtraining} by up to 5.58 and 3.31 points, respectively. Finetuning with \sysname further improves base model quality; it also enables KV-shareability, improving shared Core and Core-Ext by up to 8.35 and 4.56 points, respectively. Overall, \sysname allows looped architectures to scale recurrences to improve quality without scaling memory overhead at test-time.

\paragraph{\sysname Matches Standard Finetuning and Gains KV-Shareability.}
We further compare \sysname to standard finetuning (i.e., without an auxiliary loss term). As seen in \cref{tab:quality-ft}, \sysname achieves comparable quality to standard finetuning when no KV-cache reuse is done (i.e., \emph{unshared}). However, \sysname enables KV-shareability, reducing validation perplexity by up to 45.2\% (9.75 PPL points), when KV-caches are reused across recurrences.

\begin{table}[!t]
    \centering
    \setlength{\tabcolsep}{4pt} %
    \renewcommand{\arraystretch}{1.1} %
    \small
   \begin{tabular}{cl|ccc|ccc|ccc}
        \toprule
        & \multirow{2}{*}{\textbf{Model}} & \multicolumn{3}{c}{\textbf{Val. PPL} ($\downarrow$)} & \multicolumn{3}{c}{\textbf{Core} ($\uparrow$)} & \multicolumn{3}{c}{\textbf{Core-Extended} ($\uparrow$)} \\
        \cmidrule(lr){3-5} \cmidrule(lr){6-8} \cmidrule(lr){9-11}
        & & Unshared & Shared & $\mathbf{\Delta}$ & Unshared & Shared & $\mathbf{\Delta}$  & Unshared & Shared & $\mathbf{\Delta}$  \\
        \midrule
        \multirow{2}{*}{\rotatebox[origin=c]{90}{770M}}
& Standard FT & 11.86 & 21.61 & +9.75 & 24.74 & 23.42 & -1.32 & 15.34 & 14.54 & -0.80 \\
& \textbf{\sysname} & \textbf{11.85} & \textbf{11.86} & \textbf{+0.01} & \textbf{25.15} & \textbf{25.19} & \textbf{+0.04} & \textbf{15.35} & \textbf{15.29} & \textbf{-0.06} \\
            \cdashline{1-11}
        \multirow{2}{*}{\rotatebox[origin=c]{90}{1.3B}}
& Standard FT & \textbf{10.79} & 18.52 & +7.73 & \textbf{29.01} & 27.49 & -1.51 & 17.23 & 16.66 & -0.58 \\
& \textbf{\sysname} & \textbf{10.79} & \textbf{10.78} & \textbf{-0.01} & 28.63 & \textbf{28.67} & \textbf{+0.04} & \textbf{17.32} & \textbf{17.41} & \textbf{+0.09} \\
        \bottomrule
    \end{tabular}
    \caption{\textbf{Comparing \sysname with Standard Finetuning.} We finetune public checkpoints of Parcae \cite{prairie2026parcaescalinglawsstable} on three different seeds, evaluating on \texttt{fineweb-edu} \cite{penedo2024finewebdatasetsdecantingweb} validation and Core and Core Extended \cite{li2025datacomplmsearchgenerationtraining}. \sysname matches the quality of standard finetuning, while enabling KV-shareability. }
    \label{tab:quality-ft}
    \vspace{-0.7em}
\end{table}

\begin{figure}[!t]
    \centering
    \includegraphics[width=\linewidth]{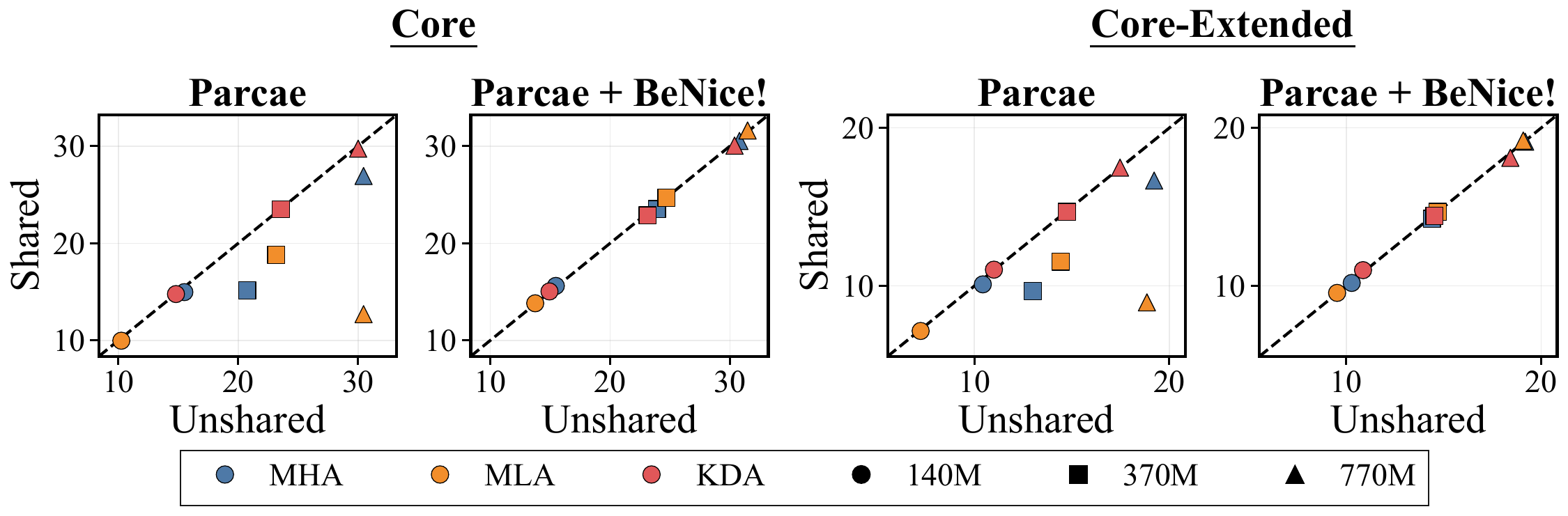}
    \vspace{-1.7em}
    \caption{\textbf{KV-Shareability with Memory-Efficient Sequence Mixers.}
    We train Parcae \cite{prairie2026parcaescalinglawsstable} models with
    MHA \cite{vaswani2023attentionneed},
    MLA \cite{deepseekai2024deepseekv2strongeconomicalefficient}, and
    KDA \cite{kimiteam2025kimilinearexpressiveefficient}. We find that \sysname allows memory-efficient sequence mixers to have KV-shareability, reducing inference overhead.}
    \label{fig:quality-mem}
    \vspace{-0.7em}
\end{figure}

\subsection{KV-Shareability is Complementary to Other KV Compression Methods.}
\label{sec:reduce-costs}

In this section, we highlight how KV-shareability is complementary to alternative KV compression methods, by composing it with memory-efficient sequence mixers (e.g., MLA \cite{deepseekai2024deepseekv2strongeconomicalefficient} and KDA \cite{kimiteam2025kimilinearexpressiveefficient}) and KV-quantization.

\paragraph{Setup.} For experiments on KV-compression, we use Parcae models from \cref{sec:quality}. For experiments composing \emph{KV-sharing} with memory-efficient sequence mixers, we train models with multi-head attention (MHA) \cite{vaswani2023attentionneed}, multi-latent attention (MLA) \cite{deepseekai2024deepseekv2strongeconomicalefficient} and Kimi delta attention (KDA) \cite{kimiteam2025kimilinearexpressiveefficient} (see \cref{sec:extended-mla-kda} for more details). For experiments composing \emph{KV-sharing} with KV-quantization, we directly reduce the precision of keys and values for MHA \cite{vaswani2023attentionneed} or the shared latent for MLA \cite{deepseekai2024deepseekv2strongeconomicalefficient}.

\paragraph{KV-Sharing Composes with Memory-Efficient Sequence Mixers.} In \cref{fig:quality-mem}, base MLA and KDA\footnote{Interestingly, KDA models appear to not require additional finetuning with \sysname.} variants of Parcae fail to have KV-shareability, suffering quality drops when sharing final recurrence KV across earlier recurrences. These KV-sharing failures mirror the same pattern as the released MHA Parcae models explored in \cref{sec:elucdiation} (see \cref{sec:extended-mla-kda} for more details).
\sysname achieves KV-shareability with all sequence mixers. Across all model scales and sequence mixers, KV-sharing decreases Core and Core-Extended \cite{li2025datacomplmsearchgenerationtraining} by 0.09 and 0.04 points, respectively, on average.

\begin{figure}[!t]
    \centering
    \includegraphics[width=\linewidth]{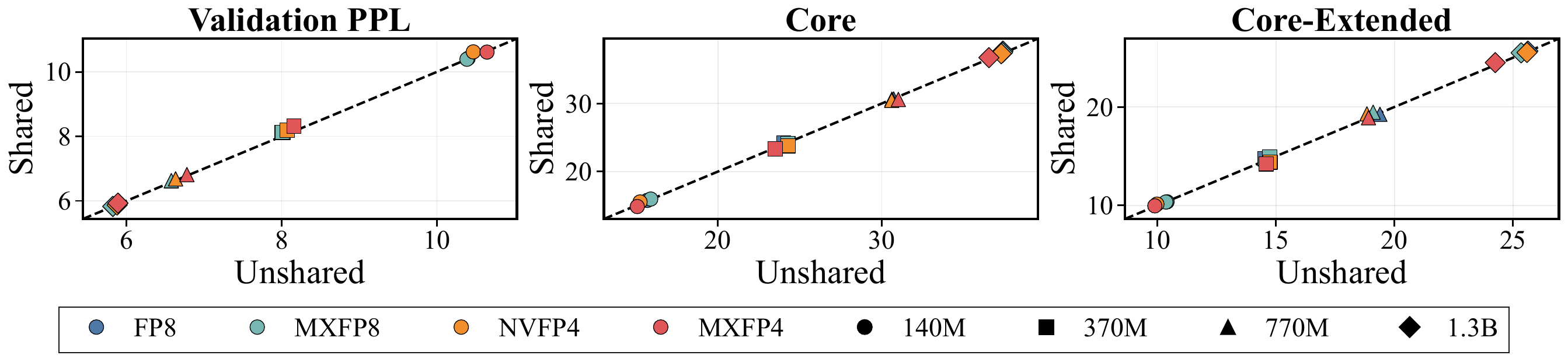}
    \vspace{-1.8em}
    \caption{\textbf{Quality of KV-Sharing and -Quantization.} We report changes in \texttt{climbmix} PPL. \cite{diao2025nemotronclimbclusteringbasediterativedata}, Core, and Core-Extended \cite{li2025datacomplmsearchgenerationtraining} after quantizing the KV-shared cache across MHA \cite{vaswani2023attentionneed} model scales. KV-sharing maintains quality across cache precisions.}
    \vspace{-1em}
    \label{fig:quantization-quality}
\end{figure}

\paragraph{KV-Sharing is Compatible with KV-Quantization.}
In \cref{fig:quantization-quality}, across scales and precisions, KV-sharing changes Core and Core-Ext by only $-0.02$ and $+0.06$ points on average, respectively, while increasing ClimbMix validation perplexity~\cite{diao2025nemotronclimbclusteringbasediterativedata} by just $0.67\%$ on average.
KV-sharing and quantization interact by at most 0.53\% in \texttt{climbmix} perplexity \cite{diao2025nemotronclimbclusteringbasediterativedata}, indicating that KV-sharing is an orthogonal compression axis.
We expect stronger KV-quantization methods (e.g., KVQuant \cite{hooper2025kvquant10millioncontext}, KIVI \cite{kivi}, ScaleSearch \cite{gupta2026searchblockfloatingpoint}) will compose with KV-sharing, further improving quality.

\subsection{Saving Memory and Accelerating Inference Workloads with KV-Shareability}
\label{sec:inference-workloads}

In this section, we simulate how memory savings from KV-shareability improve TTFT and throughput under long-context inference workloads.

\paragraph{Setup.} For long-context memory experiments, we implement efficient KV-sharing for Parcae \cite{prairie2026parcaescalinglawsstable} in \texttt{sglang} \cite{zheng2024sglangefficientexecutionstructured}. 
We report total GPU memory usage, including model weights, KV-cache storage, CUDA graphs, and runtime buffers. 
For workload experiments, we replay production-derived Mooncake requests \citep{qin2025mooncakekvcachecentricdisaggregatedarchitecture} using \texttt{AIPerf} \cite{ren2021aiperfautomatedmachinelearning}, preserving arrival times, input/output lengths, and prefix-sharing relations.
To study realistic model scales, we construct proxy models from standard open-weights (e.g., \texttt{GLM 4.7-FP8} \cite{5team2025glm45agenticreasoningcoding}) by repeatedly looping the middle third transformer blocks, with and without KV-sharing. 
We vary both the offered request rate and maximum concurrency to characterize performance across load- and concurrency-limited regimes, respectively.

\paragraph{KV-Sharing Reduces Memory Overhead.} 
In \cref{fig:memory-reduction}, we observe that memory savings scale with recurrence count ($\mu_{\text{rec}}$) and KV-cache size. For our 1.3B Parcae on Core-Extended \cite{li2025datacomplmsearchgenerationtraining} with batch size and context length of 64 and 2048, respectively, KV-shareability reduces measured global HBM usage by 66.2\% from 63.82 GB to 21.60 GB. At $\mu_{\text{rec}}=8$, KV-sharing increases the total token capacity to $3.33\times$ the unshared capacity under the same KV-memory budget.

\begin{figure}[!t]
    \vspace{-1em}
    \begin{minipage}{0.5\linewidth}
        \centering
        \includegraphics[width=\linewidth]{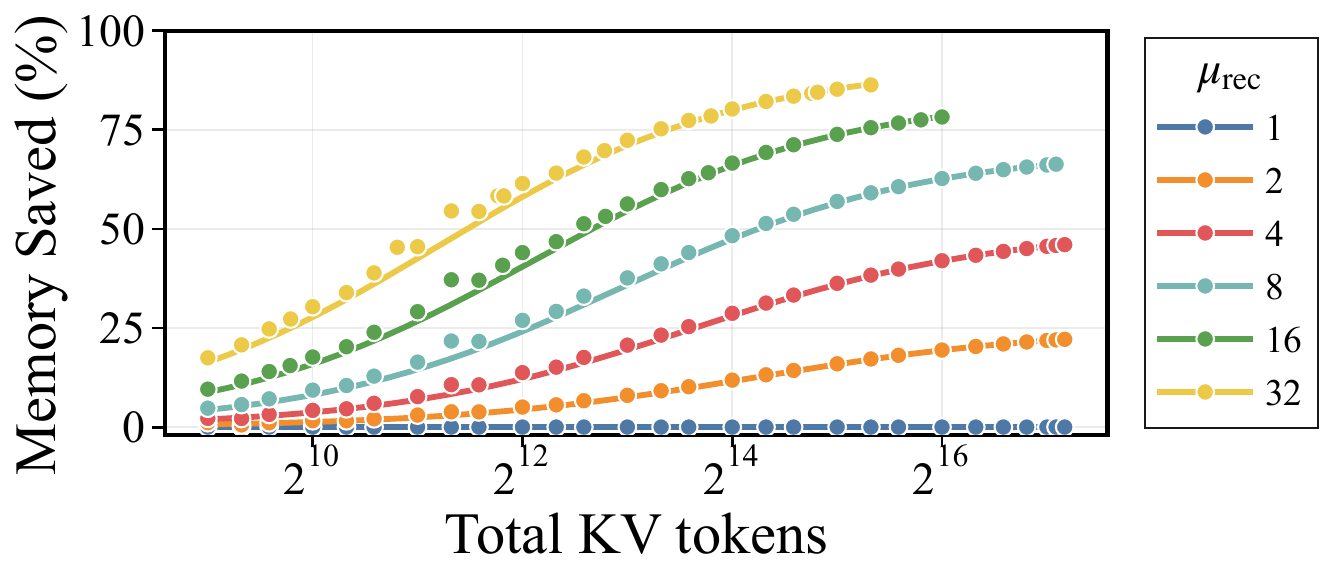}
        \vspace{-2em}
        \caption{\textbf{Global Memory Reduction of KV-Sharing.} We measure and fit the memory overhead of KV-sharing, varying test-time recurrences.}
        \label{fig:memory-reduction}
    \end{minipage}
    \hfill
    \begin{minipage}{0.49\linewidth}
        \centering
        \vspace{0.3em}
        \includegraphics[width=\linewidth]{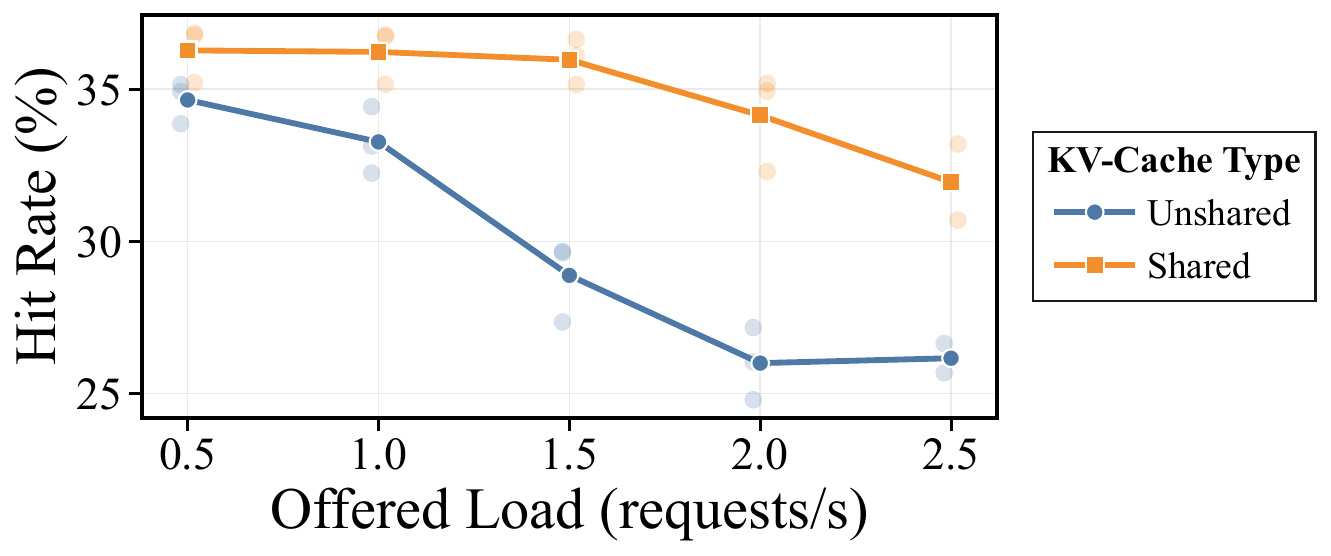}
        \vspace{-1.8em}
        \caption{\textbf{Cache-Hit Rate Under Sharing.} We average the cache-hit rate with and without KV-sharing using Mooncake traces \cite{qin2025mooncakekvcachecentricdisaggregatedarchitecture}.}
        \label{fig:cache-hit-rate}
    \end{minipage}
    \vspace{-0.7em}
\end{figure}

\begin{figure}[!t]
    \centering
    \includegraphics[width=\linewidth]{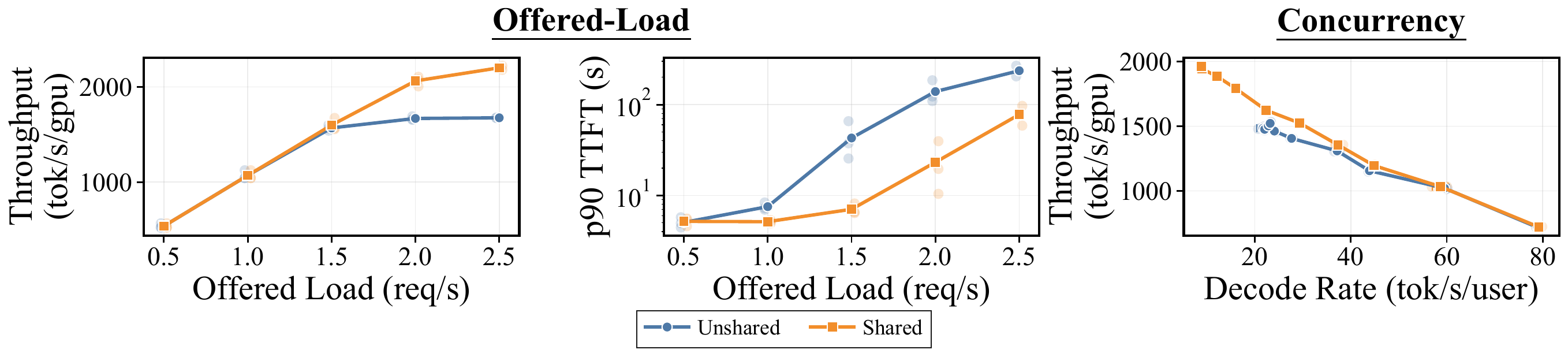}
    \vspace{-1.5em}
    \caption{\textbf{Throughput Frontier of KV-Sharing.} (\emph{Left}) Under increasing offered load, \emph{KV-Sharing} improves throughput and reduces TTFT over unshared looping. (\emph{Right}) Under high concurrency load, KV-sharing improves overall throughput at matched decode rates.}
    \vspace{-0.7em}
    \label{fig:throughput-frontier}
\end{figure}

\paragraph{KV-Sharing Accelerates Inference Workloads.}
With reduced memory overhead, KV-shareability increases the total tokens retained under a fixed cache budget (e.g., KV-sharing increases total token budgets by $3.33\times$ and $6\times$ for $\mu_{\text{rec}}=8$ and $\mu_{\text{rec}}=16$, respectively). 
Larger token budgets reduce prefix evictions and prefix recomputation (i.e., directly increasing cache hit rate by 8.1 points at $2$ req/s in \cref{fig:cache-hit-rate}).
Higher cache-hit rates further improve throughput: at low loads (e.g., 0.5 req/s), KV-sharing matches the unshared baseline in TTFT and throughput; however, at high loads (e.g., 2 req/s), KV-shareability improves throughput by 23.6\% and reduces p90 TTFT from 139.2\,s to 23.2\,s (\cref{fig:throughput-frontier} [\emph{left}]).
Along the concurrency frontier, KV-shareability provides up to 6.7\% higher throughput at a fixed median decode rate and 34.8\% higher median decode rate at a fixed throughput (\cref{fig:throughput-frontier} [\emph{right}]).
Overall, KV-shareability improves throughput and latency when workloads approach saturation.

\section{Conclusion}
\label{sec:conclusion}

In this work, we study the emergence and failure points of \emph{KV-shareability}, finding that contractive residual dynamics drive KV-caches towards a common shareable limit. To enable \emph{KV-shareability}, we propose a two stage method: first, we use stochastic depth to induce contractive residual dynamics across recurrences; second, we provide \sysname, a post-training method that corrects residual errors in the first recurrence to ensure that models approach a common limit with and without KV-sharing. KV-shareability composes with both memory-efficient sequence mixers and KV-quantization, detaches inference memory overhead from test-time recurrences, and improves throughput and TTFT in memory-constrained settings.

\paragraph{Acknowledgments.} We graciously acknowledge the support of Together AI and AMD through the FAR research grant. This research was supported in part by the NSF under grants 2146151 and 2422214.

\bibliographystyle{plainnat}
\bibliography{main}

\appendix

\newpage
\section{Glossary}
\label{sec:glosssary}
We include a brief glossary of notations present in this work in \cref{tab:glossary}.

\begin{table}[h]
    \centering
    \begin{tabular}{lll}
    \toprule
         Notation &  Description \\
    \midrule
         $d_h$ & Hidden \textbf{dimension} of the model & \\
         $t$ & Discrete temporal \textbf{state} axis of $\mathcal{R}$ on $\mathbb{N}$& \\
         $\gamma$ & Hyperparameter used to select the strength of the auxiliary loss in \sysname& \\
    \midrule
         $\mathcal{P}$ & Initial \textbf{prelude} block of a recurrent architecture & \\
         $\mathcal{R}$ & Middle \textbf{recurrent} block of a recurrent architecture & \\
         $\mathcal{C}$ & Final \textbf{coda} block of a recurrent architecture & \\
         $s$ & Input sequence to a model & \\
         $e$ & Output embedding of the prelude block $\mathcal{P}$ \\
         $h$ & Hidden embedding of the recurrent block $\mathcal{R}$\\
         $\meanrecurrence$ & Mean recurrent forward propagation steps during pre-training \\
         $\meanbackward$ & Mean recurrent backward propagation steps during pre-training \\
         $T$ & Final number of recurrent steps actually taken \\
         $\Lambda$ & Distribution that recurrences are sampled from during training \\
    \midrule
        \emph{KV-sharing} & Using final recurrence’s KV-cache in early recurrences without quality loss&\\
        \emph{Unshared} & Decoding performed with a KV cache per recurrence & \\
        \emph{Shared} & Decoding performed using only the final recurrence KV cache (KV-sharing)& \\
        $p_\theta^{(1 \leftarrow T)}$ & Output with first recurrence attending to final recurrence KV-cache & \\
        $p_\theta^{(T)}$ & Output with first recurrence attending to first recurrence KV-cache (\emph{Unshared}) & \\
        
    \bottomrule
    \end{tabular}
    \caption{Glossary of notation and terminology. (\textit{Top}) Frequently used dimensions for tensors. (\textit{Middle}) Definition of Parcae blocks, tensors, and distributions used to express recurrent depth models. (\textit{Bottom}) KV-cache compression definitions.} 
    \label{tab:glossary}
\end{table}

\section{Extended Literature Review}
\label{sec:extended-lit}

Looped architectures have been and continue to be explored by a growing body of work. We focus on architectural differences. Here, we discuss the main categories of architectural differences, as well as other works on memory efficiency in looped models.

\paragraph{Middle-Looped Models.} This framework of models utilizes three blocks: a prelude block ($\mathcal{P}$) that performs representation extraction using transformer layers, a recurrent block ($\mathcal{R}$) that recurrently updates activations, and a coda block ($\mathcal{C}$) that turns activations into outputs. \citet{bansalEndtoendAlgorithmSynthesis2022}, \citet{geiping_scaling_2025}, \citet{mcleish_retrofitted_recurrence}, and \citet{prairie2026parcaescalinglawsstable} all train middle-looped models with stochastically sampled recurrence depth during pretraining. \citet{Schwethelm2026-hv} trained middle-looped models with fixed recurrence depth. These architectures tend to be trained with implicit halting mechanisms to determine when to stop looping \cite{bansalEndtoendAlgorithmSynthesis2022, geiping_scaling_2025, mcleish_retrofitted_recurrence, prairie2026parcaescalinglawsstable, Schwethelm2026-hv}, using either convergence methods like KL divergence between loops \cite{bansalEndtoendAlgorithmSynthesis2022, mcleish_retrofitted_recurrence}, low relative differences in successive hidden states \cite{geiping_scaling_2025}, or use the mean training recurrence as the optimal inference time recurrence depth \cite{prairie2026parcaescalinglawsstable, Schwethelm2026-hv}. However, \citet{baeMixtureofRecursionsLearningDynamic2025} is trained with a middle-looped architecture, but uses an explicit halting mechanism of lightweight routers to decide recurrence depth.

\paragraph{Fully Looped Models.} This regime of looped models loops the entire model every recurrence. While not strictly a requirement, this regime of architecture tends to pull from \citet{graves2017adaptivecomputationtimerecurrent} work on adaptive compute time (ACT), using explicit halting mechanism to determine when to stop looping. \citet{dehghaniUniversalTransformers2019}, \citet{zhuScalingLatentReasoning2025} (Ouro), \citet{elbayadDepthAdaptiveTransformer2020}, \citet{jolicoeur-martineauLessMoreRecursive2025},  and \citet{Fu2026-rz} all extend this concept to language models, varying the mechanisms used for explicit halting. These halting mechanisms include prediction heads for dynamically deciding halting probabilities \cite{zhuScalingLatentReasoning2025, elbayadDepthAdaptiveTransformer2020} and Q-learning as a means of determining convergence \cite{jolicoeur-martineauLessMoreRecursive2025, dehghaniUniversalTransformers2019, graves2017adaptivecomputationtimerecurrent}. Most of these models are trained with fixed recurrence depth \cite{dehghaniUniversalTransformers2019, zhuScalingLatentReasoning2025, elbayadDepthAdaptiveTransformer2020, Fu2026-rz}.

\paragraph{KV-sharing and Memory Efficiency in Looped Models.} Huginn \cite{geiping_scaling_2025, geiping2025efficientparallelsamplersrecurrentdepth} was noted as having KV-sharing ability out-of-the-box. \citet{baeMixtureofRecursionsLearningDynamic2025} proposed a KV-sharing variant for their model that shares first recurrence KV-cache, with some quality loss. \citet{Deng2026-pw} extends the Ouro\cite{zhuScalingLatentReasoning2025} style looped model to use with memory-efficient sequence mixers. \citet{Vendrell2026-og} proposed a KV-sharing variant of the Ouro style model that shares the first recurrence KV-cache with later recurrences in addition to maintaining a sliding window KV-cache per recurrence.

\section{Evaluation Setup of KV-Shareability}
\label{sec:eval-setup}

In this section, we will discuss how we augment benchmarks to evaluate whether a looped architecture can losslessly utilize \emph{KV-sharing}. We will break benchmarks into two types (i.e., perplexity- and question-answering-based evaluations) and discuss how each is evaluated below.

\paragraph{Perplexity-Based Evalutations.} For a standard perplexity benchmark $S$ evaluated on a looped architecture with $T$ test-time recurrences, the average per-token perplexity is

\begin{equation}
    \mathrm{PPL}(S, T) \;=\;
        \exp\!\left( -\frac{1}{\sum_{s \in S} |s|} \sum_{s \in S} \sum_{i=1}^{|s|}
            \log p_\theta^{(T)}\!\left(s_i \mid s_{<i}\right) \right),
\end{equation}

where $p_\theta^{(T)}$ denotes the model's next-token distribution when run for $T$ recurrences. Because the full prefix $s_{<i}$ can be teacher-forced in a single parallel forward pass, the per-sequence cost is
$\mathcal{O}(T\,|s|)$.
To assess whether a looped architecture has \emph{KV-shareability}, we
augment this benchmark in two ways. First, we replace the per-recurrence
KV cache with the projection
\begin{equation*}
    \Pi_{\mathrm{shared}}\!\left(\bigl\{K_\ell^{(t)},\, V_\ell^{(t)}\bigr\}_{t=1}^{T}\right)
    \;:=\;
    \bigl\{K_\ell^{(T)},\, V_\ell^{(T)}\bigr\}_{t=1}^{T},
    \qquad \forall\, \ell \in \mathcal{R},
\end{equation*}
which broadcasts the \emph{final} recurrence's keys and values across all $T$ recurrences in the looped block. To recover the standard perplexity evaluation, we can use the operator $\Pi_\text{unshared}$, where 
\begin{equation*}
    \Pi_{\text{unshared}}\!\left(\bigl\{K_\ell^{(t)},\, V_\ell^{(t)}\bigr\}_{t=1}^{T}\right)
    \;:=\;
    \bigl\{K_\ell^{(t)},\, V_\ell^{(t)}\bigr\}_{t=1}^{T},
    \qquad \forall\, \ell \in \mathcal{R}.
\end{equation*}

Finally, as populating the shared cache $\Pi_T$ for the prediction at position $i$ requires re-running the entire prefix token-by-token (each token must complete all $T$ recurrences so its final-recurrence $K^{(T)}_\ell, V^{(T)}_\ell$ is written into the shared slot before the next token can read it back), the evaluation no longer admits a single parallel forward pass per sequence, giving
\begin{equation}
    \mathrm{PPL}(S, T;\, \Pi_\mathrm{shared}) \;=\;
        \exp\!\left( -\frac{1}{\sum_{s \in S} |s|} \sum_{s \in S} \sum_{i=1}^{|s|}
            \log p_\theta^{(T)}\!\left(s_i \,\Big|\, \Pi_\mathrm{shared}\!\left(s_{<i}\right)
            \right) \right),
\end{equation}

where $\Pi_\mathrm{shared}(s_{<i})$ denotes the prefix $s_{<i}$ processed under the KV-cache reuse projection $\Pi_\mathrm{shared}$ (i.e., with each token's per-recurrence keys and values $\{K^{(t)}_\ell, V^{(t)}_\ell\}_{t=1}^{T}$ collapsed to the final-recurrence cache $\{K^{(T)}_\ell, V^{(T)}_\ell\}_{t=1}^{T}$ before the next token attends back to it). The per-sequence cost is $\mathcal{O}(T\,|s|^2)$ as one factor of $|s|$ comes from the outer position sum, and another from the prefix that must be re-processed for each prediction.

\paragraph{Question-Answering Evaluations.}
For a question prompt $q$ with candidate answers
$\mathcal{A}=\{a^{(1)},\ldots,a^{(m)}\}$, we rank candidates by their
length-normalized conditional negative log-likelihood:
\begin{equation}
    \mathcal{L}_{\Pi}\!\left(a^{(j)} \mid q\right)
    =
    -\frac{1}{|a^{(j)}|}
    \sum_{k=1}^{|a^{(j)}|}
    \log p_\theta^{(T)}\!\left(
        a_k^{(j)}
        \,\middle|\,
        \Pi\!\left(q \circ a_{<k}^{(j)}\right)
    \right),
\end{equation}
and predict
\begin{equation}
    \hat{y}
    =
    \arg\min_{j \in \{1,\ldots,m\}}
    \mathcal{L}_{\Pi}\!\left(a^{(j)} \mid q\right).
\end{equation}
For KV-cache reuse, we prefill the common prompt $q$ once under
$\Pi_{\mathrm{shared}}$, replicate the resulting cache across candidates,
and decode each candidate token-by-token so that every token writes its
final-recurrence keys and values before the next token is scored.
For the unshared baseline, all candidate sequences are evaluated with a
standard parallel teacher-forced forward pass.

\paragraph{Categorization of Benchmarks.}
Table~\ref{tab:eval-tasks} summarizes the 46 downstream tasks used in our evaluation, grouped by domain, scoring protocol, and number of in-context examples. CORE comprises the marked 22-task subset, while CORE-Extended includes all 46 tasks.

\clearpage

\begin{table*}[t]
\centering
\setlength{\tabcolsep}{4pt}
\renewcommand{\arraystretch}{0.92}
\scriptsize
\begin{tabular}{@{}llcccc@{}}
\toprule
\textbf{Category} & \textbf{Task} & \textbf{Type} & \textbf{Shots}
& \textbf{KV Protocol} & \textbf{CORE} \\
\midrule
\multirow{8}{*}{\rotatebox[origin=c]{90}{\textit{Understanding}}}
  & HellaSwag (0-shot) \cite{zellers2019hellaswag} & MC & 0 & QA & \cmark \\
  & HellaSwag (10-shot) \cite{zellers2019hellaswag} & MC & 10 & QA & \cmark \\
  & LAMBADA \cite{paperno2016lambada} & LM & 0 & QA/PPL & \cmark \\
  & Winograd WSC \cite{wsc:2015} & S & 0 & QA & \cmark \\
  & WinoGrande \cite{sakaguchi2021winogrande} & S & 0 & QA & \cmark \\
  & BIG-Bench Language ID \cite{srivastava2023imitationgamequantifyingextrapolating} & MC & 10 & QA & \cmark \\
  & BIG-Bench Conlang Translation \cite{srivastava2023imitationgamequantifyingextrapolating} & LM & 0 & QA & \\
  & BIG-Bench Conceptual Combinations \cite{srivastava2023imitationgamequantifyingextrapolating} & MC & 10 & QA & \\
\midrule
\multirow{7}{*}{\rotatebox[origin=c]{90}{\textit{World Knowledge}}}
  & Jeopardy \cite{kaggle200000Jeopardy} & LM & 10 & QA & \cmark \\
  & BIG-Bench QA WikiData \cite{srivastava2023imitationgamequantifyingextrapolating} & LM & 10 & QA & \cmark \\
  & ARC-Easy \cite{clark2018think} & MC & 10 & QA & \cmark \\
  & ARC-Challenge \cite{clark2018think} & MC & 10 & QA & \cmark \\
  & MMLU (0-shot) \cite{hendrycks2021measuringmassivemultitasklanguage} & MC & 0 & QA & \\
  & MMLU (5-shot) \cite{hendrycks2021measuringmassivemultitasklanguage} & MC & 5 & QA & \\
  & BIG-Bench Misconceptions \cite{srivastava2023imitationgamequantifyingextrapolating} & MC & 10 & QA & \\
\midrule
\multirow{8}{*}{\rotatebox[origin=c]{90}{\textit{Commonsense}}}
  & COPA \cite{gordon-etal-2012-semeval} & MC & 0 & QA & \cmark \\
  & CommonsenseQA \cite{talmor2019commonsenseqaquestionansweringchallenge} & MC & 10 & QA & \cmark \\
  & PIQA \cite{bisk2020piqa} & MC & 10 & QA & \cmark \\
  & OpenBookQA \cite{mihaylov2018suitarmorconductelectricity} & MC & 0 & QA & \cmark \\
  & SIQA \cite{sap2019socialiqacommonsensereasoningsocial} & MC & 10 & QA & \\
  & BIG-Bench Novel Concepts \cite{srivastava2023imitationgamequantifyingextrapolating} & MC & 10 & QA & \\
  & BIG-Bench Strange Stories \cite{srivastava2023imitationgamequantifyingextrapolating} & MC & 10 & QA & \\
  & BIG-Bench StrategyQA \cite{srivastava2023imitationgamequantifyingextrapolating} & MC & 10 & QA & \\
\midrule
\multirow{11}{*}{\rotatebox[origin=c]{90}{\textit{Symbolic/Math}}}
  & BIG-Bench Dyck Languages \cite{srivastava2023imitationgamequantifyingextrapolating} & LM & 10 & QA & \cmark \\
  & AGI Eval LSAT AR \cite{zhong2021arlsat} & MC & 3 & QA & \cmark \\
  & BIG-Bench CS Algorithms \cite{srivastava2023imitationgamequantifyingextrapolating} & LM & 10 & QA & \cmark \\
  & BIG-Bench Operators \cite{srivastava2023imitationgamequantifyingextrapolating} & LM & 10 & QA & \cmark \\
  & BIG-Bench Repeat Copy Logic \cite{srivastava2023imitationgamequantifyingextrapolating} & LM & 10 & QA & \cmark \\
  & BIG-Bench Elementary Math QA \cite{srivastava2023imitationgamequantifyingextrapolating} & MC & 10 & QA & \\
  & BIG-Bench Logical Deduction \cite{srivastava2023imitationgamequantifyingextrapolating} & MC & 10 & QA & \\
  & Simple Arithmetic (no spaces) \cite{llmfoundry} & LM & 10 & QA & \\
  & Simple Arithmetic (with spaces) \cite{llmfoundry} & LM & 10 & QA & \\
  & MathQA \cite{amini2019mathqa} & MC & 10 & QA & \\
  & LogiQA \cite{liu2020logiqa} & MC & 10 & QA & \\
\midrule
\multirow{8}{*}{\rotatebox[origin=c]{90}{\textit{Reading Comp.}}}
  & SQuAD \cite{rajpurkar2016squad100000questionsmachine} & LM & 10 & QA & \cmark \\
  & CoQA \cite{reddy2019coqaconversationalquestionanswering} & LM & 0 & QA & \cmark \\
  & BoolQ \cite{clark2019boolq} & MC & 10 & QA & \cmark \\
  & PubMedQA (labeled) \cite{jin2019pubmedqa} & LM & 10 & QA & \\
  & AGI Eval LSAT RC \cite{zhong2023agieval} & MC & 3 & QA & \\
  & AGI Eval LSAT LR \cite{wang2022lsat} & MC & 3 & QA & \\
  & AGI Eval SAT English \cite{zhong2023agieval} & MC & 3 & QA & \\
  & BIG-Bench Understanding Fables \cite{srivastava2023imitationgamequantifyingextrapolating} & MC & 10 & QA & \\
\midrule
\multirow{4}{*}{\rotatebox[origin=c]{90}{\textit{Safety}}}
  & Winogender MC (Female) \cite{rudinger2018genderbiascoreferenceresolution} & MC & 10 & QA & \\
  & Winogender MC (Male) \cite{rudinger2018genderbiascoreferenceresolution} & MC & 10 & QA & \\
  & Enterprise PII Classification & MC & 10 & QA & \\
  & BBQ \cite{parrish2022bbqhandbuiltbiasbenchmark} & MC & 3 & QA & \\
\bottomrule
\end{tabular}
\caption{\textbf{Downstream evaluation tasks.}
Tasks marked with \cmark{} belong to the 22-task CORE suite; all 46 tasks belong to CORE-Extended. Under the \textbf{QA} KV-sharing protocol, the context is prefilled using the final-recurrence cache and candidate continuations are scored against that cache. Under the \textbf{PPL} protocol, used for the separate LAMBADA BPB evaluation, the complete sequence is streamed token-by-token so that every prediction reads only final-recurrence KVs. Type denotes the task scoring rule: \textbf{MC} uses mean candidate NLL, \textbf{S} uses schema-based NLL, and \textbf{LM} requires an exact greedy continuation match.}
\label{tab:eval-tasks}
\end{table*}

\clearpage

\newpage
\section{Proof of KV Convergence}
\label{sec:convergence-of-attention}

In this section, we make precise the claim in \cref{sec:elucdiation} that fixed-point convergence of the hidden state induces convergence of the per-layer keys and values. Specifically, we show that if the recurrent block $\rec$ is a contraction with a fixed point $\hat h$, then for every attention layer $l \in \rec$, then $k_t^{(l)},v_t^{(l)} \to \hat k^{(l)}, \hat v^{(l)}$ at a rate controlled by the step-to-step residual $\lVert h_t - h_{t-1}\rVert$.

Furthermore, we comment on the scope of this result. First, there is no guarantee that looped architectures such as Parcae \citep{prairie2026parcaescalinglawsstable} and Huginn \citep{geiping_scaling_2025} are contractive. However, both prior works empirically observe convergence to fixed points, and we independently verify this behavior in \cref{sec:parcae-is-fixed-point}. Second, we are not claiming that self-attention itself is Lipschitz, which has been proved not to be globally Lipschitz on unbounded inputs \cite{kim2021lipschitzconstantselfattention}. Our proof only requires that the per-layer key and value \emph{projection}, $h \mapsto W_{K,V}^{(l)} \cdot \mathrm{RMSNorm}(h)$, is Lipschitz. This follows from standard bounds on linear maps composed with the Lipschitz bound of RMSNorm \citep{zhang2019rootmeansquarelayer,chen2026histanumcaestimatestate}.

\begin{theorem}\label{thm:kv_convergence}
Let $\rec:\mathbb{R}^d\to\mathbb{R}^d$ be $\kappa$-Lipschitz with $\kappa<1$, and for each layer $l\in\rec$ let $\phi_K^{(l)},\phi_V^{(l)}:\mathbb{R}^d\to\mathbb{R}^{d_h}$ be $L^{(l)}$-Lipschitz. Let $h_{t+1}=\rec(h_t)$, $k_t^{(l)}=\phi_K^{(l)}(h_t)$, $v_t^{(l)}=\phi_V^{(l)}(h_t)$. Then $\rec$ has a unique fixed point $\hat h$, and for every $t\le T$,
$$
\bigl\lVert k_t^{(l)}-k_T^{(l)}\bigr\rVert,\ \bigl\lVert v_t^{(l)}-v_T^{(l)}\bigr\rVert \;\le\; \frac{L^{(l)}}{1-\kappa}\,\lVert h_{t+1}-h_t\rVert,
$$
with $k_t^{(l)}\to\hat k^{(l)}:=\phi_K^{(l)}(\hat h)$ and $v_t^{(l)}\to\hat v^{(l)}:=\phi_V^{(l)}(\hat h)$ as $t\to\infty$.
\end{theorem}
\begin{proof}
Lipschitz continuity of $\rec$ gives $\lVert h_{s+1}-h_s\rVert\le\kappa^{s-t}\lVert h_{t+1}-h_t\rVert$ for $s\ge t$. Summing,
\begin{align}
\lVert h_T-h_t\rVert
&\le \sum_{s=t}^{T-1}\lVert h_{s+1}-h_s\rVert \\
&\le \lVert h_{t+1}-h_t\rVert \sum_{j=0}^{T-t-1}\kappa^{j} & \text{via Triangle Inequality}\\
&\le \lVert h_{t+1}-h_t\rVert \sum_{j=0}^{\infty}\kappa^{j} \\
&\le \frac{\lVert h_{t+1}-h_t\rVert}{1-\kappa} & \text{via Geometric Sum Closed Form}
\end{align}
Lipschitz continuity of $\phi_K^{(l)}$ (resp.\ $\phi_V^{(l)}$) then gives $\lVert k_t^{(l)}-k_T^{(l)}\rVert\le L^{(l)}\lVert h_t-h_T\rVert$, yielding the stated bound. Existence of $\hat h$ and convergence $h_t\to\hat h$ follow from Banach's fixed-point theorem, and continuity of $\phi_{K,V}^{(l)}$ transfers convergence to the keys and values.
\end{proof}

\section{KV-Sharing Error Subspace}
\label{sec:low-dim-construction}

In this section, we formalize the construction of the KV Sharing Error Subspace. This construction is inspired by \cite{Li2026-ge}.

Looking only at the first layer in the recurrent block and only during the first recurrence, let $a_\mathrm{shared} \in \mathbb{R}^{N \times d_h}$ be attention output for $N$ tokens attending to only final recurrence KV cache, while $a_{\mathrm{unshared}} \in \mathbb{R}^{N \times d_h}$ is the attention output using the first recurrence KV for the same $N$ tokens (we use $N=24$K tokens).

Then, define $\delta = a_{\mathrm{unshared}}-a_\mathrm{shared}$ as the per-token differences between the attention outputs using first recurrence KV and shared KV.

We perform SVD on $\delta$. In \cref{fig:low-rank-subspace}, we observe that variance within $\delta$ is low-rank and well explained with 8 right singular vectors, defining our KV-sharing error subspace.
$$\begin{aligned}
\delta  &= U \Sigma V^\top\in \mathbb{R}^{N\times d_h} \\
B_8 &= V_{:,1:8}
\end{aligned}$$ 
\begin{figure}
    \centering
    \includegraphics[width=0.5\textwidth]{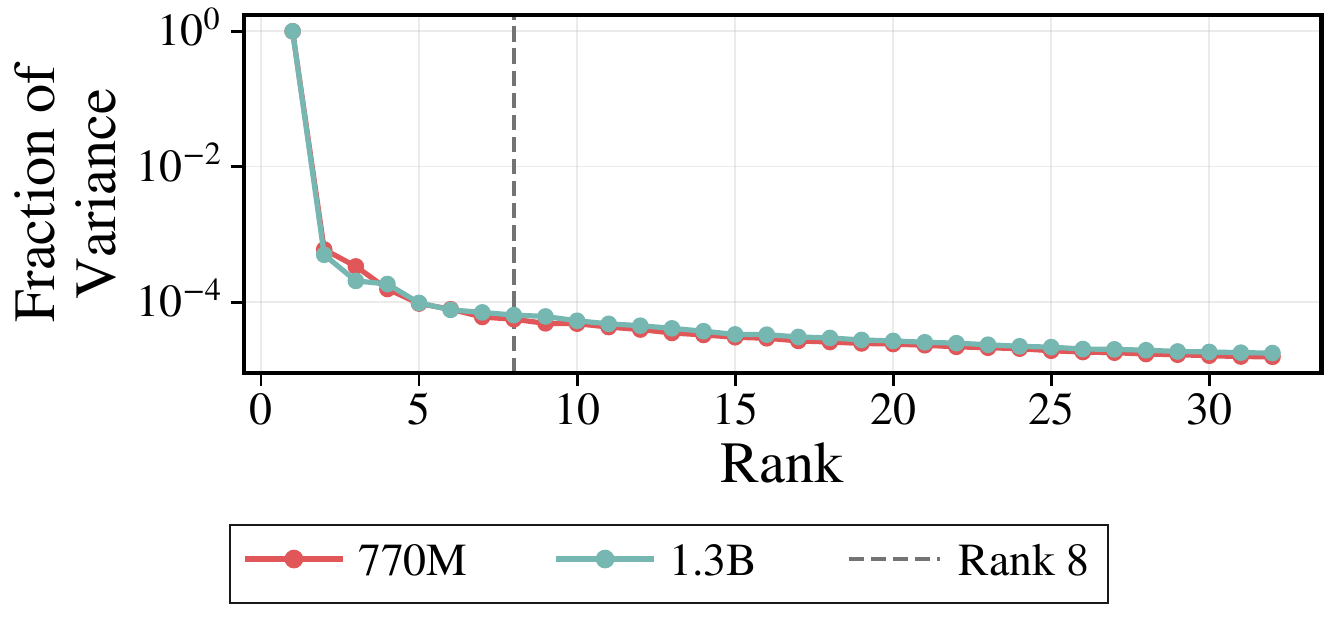}
        \vspace{-1.7em}
        \caption{\textbf{KV-sharing Error Subspace is Low Rank.} The variance within KV-sharing error subspace is well explained with a few directions.}
    \label{fig:low-rank-subspace}
\end{figure}
\paragraph{Naive Patching}
Here we provide a more formal explanation of our \emph{Naive Patching} work. We take the average across per token differences within the subspace and project it back into the hidden dimension.
\begin{equation}
    \bar{a} = \left( \frac{1}{N} \mathbf{1}^\top \delta B_8 \right)B_8^\top
\end{equation}
Then, during KV-sharing decode, we add the average projected subspace error to the first layer, first recurrence attention output per token.
\begin{equation}
    \hat{a}_{\mathrm{shared}}^t = a_\mathrm{shared}^t + \bar{a}
\end{equation}

All other layers receive no patch, and the first layer on all subsequent recurrences receives no patch. \emph{Naive Patching} is nearly lossless in terms of KV-sharing quality (\cref{tab:compression-patching}).

\section{Parcae and Huginn Converge to Fixed-Points}
\label{sec:parcae-is-fixed-point}

In this section, we provide evidence that Parcae \cite{prairie2026parcaescalinglawsstable} and Huginn \cite{geiping_scaling_2025} contract to fixed-points.

\paragraph{Setup.}
We evaluate the publicly released Parcae 140M, 370M, 770M, and 1.3B checkpoints \cite{prairie2026parcaescalinglawsstable}, together with Huginn 3.5B \cite{geiping_scaling_2025}, without additional training. For each model, we process 1,000 sequences of length 2,048 from the \texttt{climbmix} validation set in BF16. We evaluate Parcae for 15 recurrences and Huginn for its standard 32 recurrences. At each recurrence $T$, we measure the relative change in the recurrent residual state,
\begin{equation}
    \Delta_h(T)
    =
    \mathbb{E}_{x}\!\left[
    \frac{\lVert h_T-h_{T-1}\rVert_{\mathrm{RMS}}}
         {\lVert h_{T-1}\rVert_{\mathrm{RMS}}}
    \right],
\end{equation}
and the relative change in attention patterns,
\begin{equation}
    \Delta_A(T)
    =
    \mathbb{E}_{x,\ell,j}\!\left[
    \frac{\lVert A_T^{(\ell,j)}-A_{T-1}^{(\ell,j)}\rVert_F}
         {\lVert A_{T-1}^{(\ell,j)}\rVert_F}
    \right],
\end{equation}
where $\ell$ and $j$ index recurrent-block layers and attention heads, respectively.

\paragraph{Results.}
As shown in \cref{fig:baseline-fixed-point-convergence}, residual states and attention patterns stabilize with recurrence depth across all models. Parcae 140M--770M converge rapidly, while Parcae 1.3B converges more slowly; Huginn exhibits the same behavior over its 32-recurrence horizon. Fixed-point dynamics are already present in the original pretrained models and are not specific to our KV-sharing experiments. All of these models were trained with stochastic-depth.

\begin{figure*}[p]
    \centering
    \includegraphics[width=\textwidth]{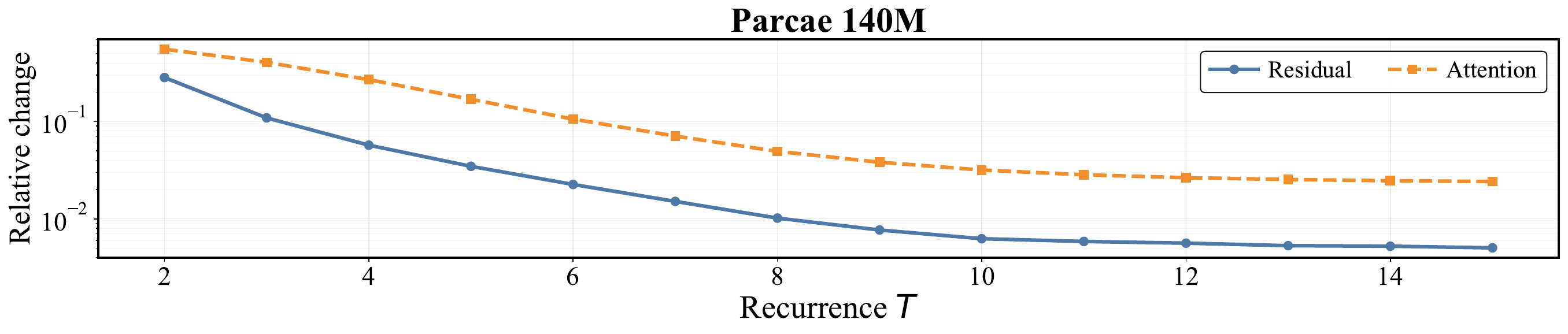}
    \includegraphics[width=\textwidth]{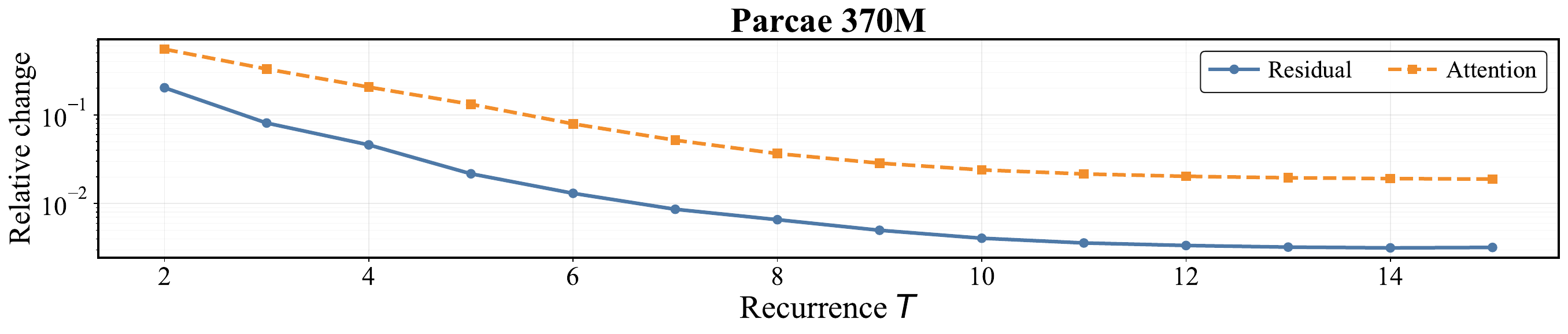}
    \includegraphics[width=\textwidth]{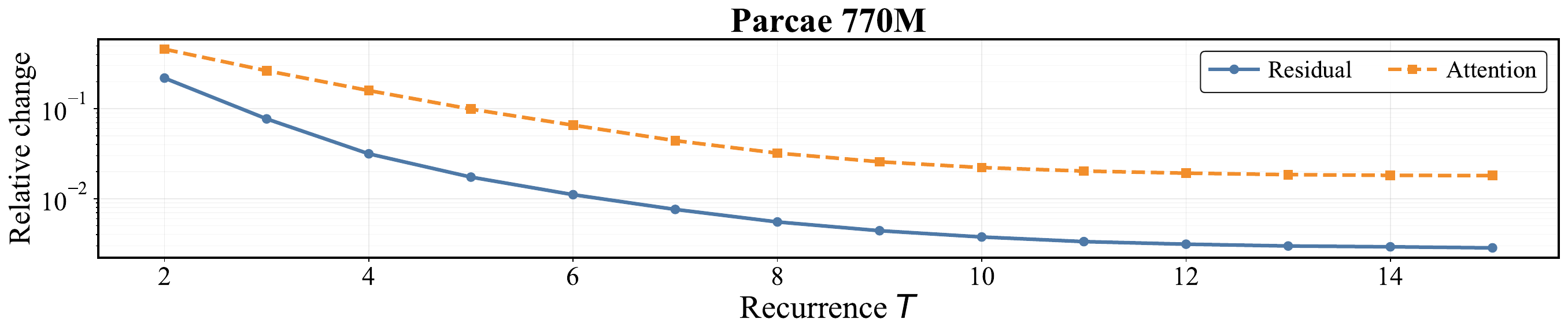}
    \includegraphics[width=\textwidth]{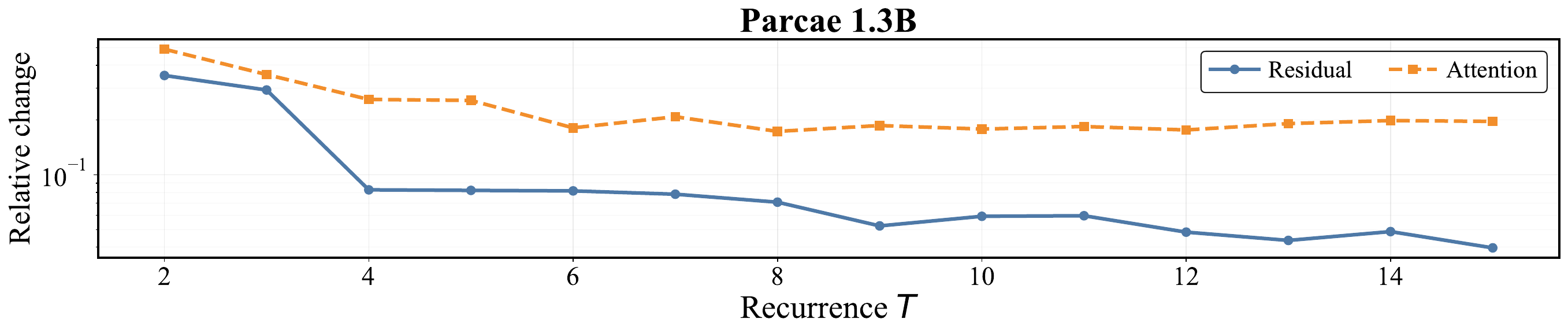}
    \includegraphics[width=\textwidth]{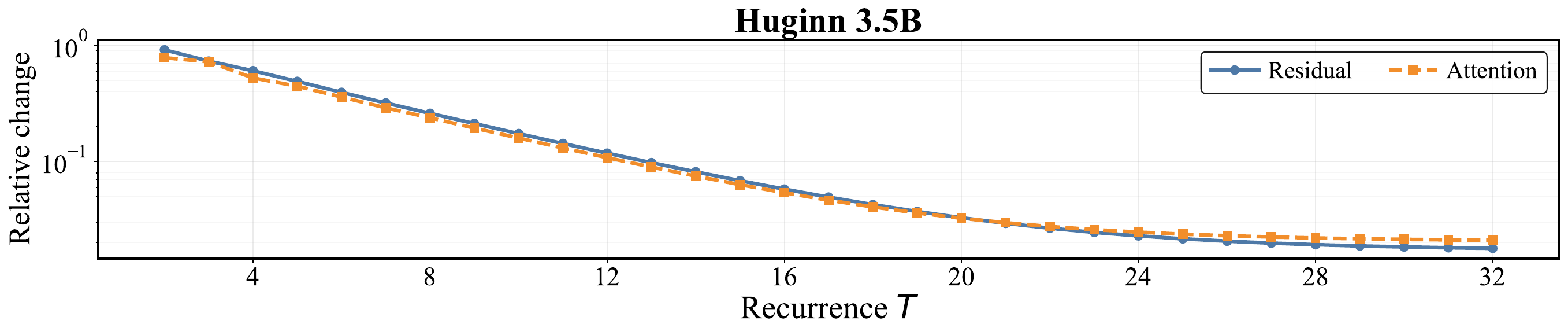}
    \caption{\textbf{Fixed-point convergence of pretrained looped models.}
    Relative changes in the recurrent residual state and attention pattern decrease with recurrence depth for the released Parcae models and Huginn.}
    \label{fig:baseline-fixed-point-convergence}
\end{figure*}

\newpage
\section{Expanded Results Comparing Stochastic and Static Sampling Depth Training}
\label{sec:expanded-stochastic-fixed}

In this section, we investigate the emergence of KV-sharing by connecting it to fixed-point convergence in looped architectures. Specifically, we train size-matched Parcae models, using either stochastic or static recurrence depth, and compare their residual convergence, attention convergence, and quality under KV-sharing.

\paragraph{Setup.} We follow the experimental setup of \citet{prairie2026parcaescalinglawsstable}, using Parcae \cite{prairie2026parcaescalinglawsstable} as the architectural backbone for our experiments (see \cref{sec:model-configurations} for model definitions), and trained with $\meanrecurrence = 8$ and $\meanbackward = 4$. For stochastic depth training, we use a Poisson distribution following \cite{prairie2026parcaescalinglawsstable}, while for static depth training, we forward pass exactly $\meanrecurrence$ and a truncated backpropagation of exactly $\meanbackward$. We train on \texttt{climbmix} \cite{diao2025nemotronclimbclusteringbasediterativedata}, using approximately 1:7 param-to-token ratios, and using the hyperparameters defined in \cref{sec:hyperparameters}. Hyperparameters were selected after extensive tuning following \cref{sec:hyperparameter-transfer}. Evaluation is done following \cref{sec:eval-setup}.

\paragraph{Stochastic and Static Depth Training.}
We compare models trained with stochastic recurrence depth against models trained at a static depth of $T=8$. Both exhibit similar training-loss trajectories (\cref{fig:stochastic-static-training}) and comparable uncompressed quality. However, KV-sharing is nearly lossless for stochastic-depth models at every scale, while static-depth models degrade substantially under compression (\cref{tab:stochastic-static-quality}). Thus, the emergence of KV-sharing is associated with stochastic-depth training rather than ordinary language-model optimization alone.

\begin{figure}[!ht]
    \centering
    \includegraphics[width=\linewidth]{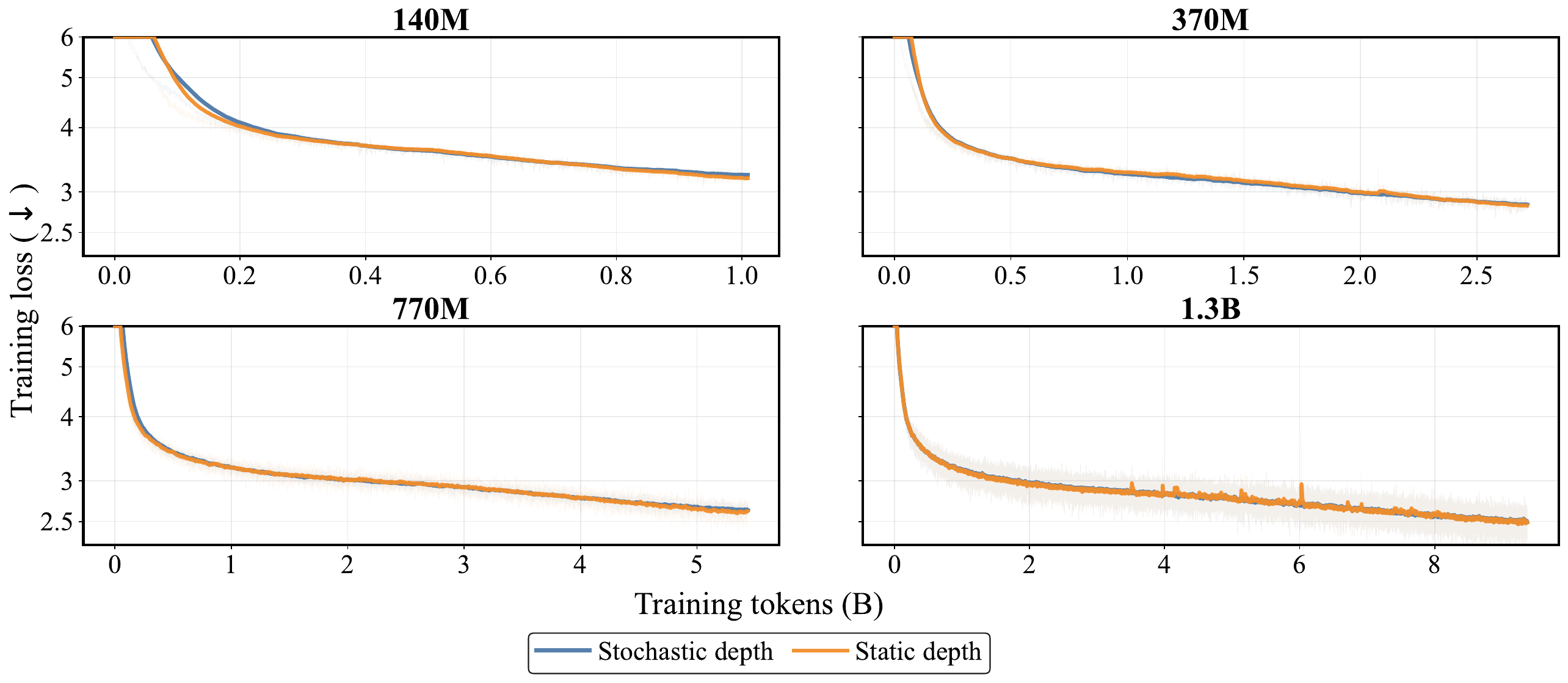}
    \caption{\textbf{Training Curves for Stochastic- and Static-Depth Models.}
    Training loss as a function of training tokens across model scales. The curves closely overlap, showing that stochastic- and static-depth training exhibit similar optimization dynamics.}
    \label{fig:stochastic-static-training}
\end{figure}

\begin{table}[!ht]
    \centering
    \setlength{\tabcolsep}{5pt}
    \renewcommand{\arraystretch}{1.15}
    \small
    \begin{tabular}{cl|ccc|ccc}
        \toprule
        & \multirow{2}{*}{\textbf{Training}} & \multicolumn{3}{c}{\textbf{Val. PPL} ($\downarrow$)} & \multicolumn{3}{c}{\textbf{CORE} ($\uparrow$)} \\
        \cmidrule(lr){3-5} \cmidrule(lr){6-8}
        & & Unshared & Shared & $\mathbf{\Delta}$ & Unshared & Shared & $\mathbf{\Delta}$ \\
        \midrule
        \multirow{2}{*}{\rotatebox[origin=c]{90}{\emph{140m}}} & Stochastic & 25.68 & 25.67 & -0.02 & 6.41 & 6.48 & +0.08 \\
        & Static & 24.48 & 25.79 & +1.31 & 9.17 & 8.55 & -0.62 \\
        \cdashline{1-8}
        \multirow{2}{*}{\rotatebox[origin=c]{90}{\emph{370m}}} & Stochastic & 17.34 & 17.36 & +0.02 & 14.49 & 14.36 & -0.13 \\
        & Static & 17.30 & 20.00 & +2.70 & 15.00 & 14.22 & -0.78 \\
        \cdashline{1-8}
        \multirow{2}{*}{\rotatebox[origin=c]{90}{\emph{770m}}} & Stochastic & 14.03 & 14.04 & +0.01 & 18.91 & 18.78 & -0.13 \\
        & Static & 14.10 & 17.20 & +3.10 & 18.17 & 17.11 & -1.06 \\
        \cdashline{1-8}
        \multirow{2}{*}{\rotatebox[origin=c]{90}{\emph{1.3b}}} & Stochastic & 12.33 & 12.34 & +0.01 & 23.63 & 23.34 & -0.29 \\
        & Static & 12.32 & 15.17 & +2.84 & 22.59 & 20.82 & -1.77 \\
        \bottomrule
    \end{tabular}
    \caption{\textbf{Quality under stochastic- and static-depth training.} We compare the standard per-recurrence cache (Unshared) with final-recurrence KV sharing (Shared). $\Delta$ is Shared minus Unshared.}
    \label{tab:stochastic-static-quality}
\end{table}

\section{Extended Empirical Analysis on KV-Sharing Fixed Point Convergence Failures}
\label{sec:fixed-point-no-good}

In this section, we provide extended empirical evidence that KV-sharing failures in models trained with stochastic depth are fixed-point convergence failures. Specifically, we provide expanded empirical evidence that while contractive residual dynamics are necessary to induce KV-shareability, models can still fail to converge to the correct fixed points.

\paragraph{Setup.} We use prior Parcae looped models \cite{prairie2026parcaescalinglawsstable} to explore our observations about KV-sharing failures. We additionally use models trained in \cref{sec:quality}, showing how \sysname impacts fixed-point convergence, before and after finetuning.
Following convergence observations made in \cref{sec:parcae-is-fixed-point} and the saturation of test-time compute made by \citet{prairie2026parcaescalinglawsstable}, we use $T=8$ test-time recurrences, as both the residual differences and downstream quality saturate at this point. We use \texttt{climbmix} \cite{diao2025nemotronclimbclusteringbasediterativedata} validation data for all evaluations and explorations.

\begin{table}[!ht]
    \centering
    \setlength{\tabcolsep}{5pt}
    \renewcommand{\arraystretch}{1.15}
    \small
    \begin{tabular}{l|ccc}
        \toprule
\textbf{Model} & $\mathbf{\mathbb{E}[\|h_8^{(\text{shared})}-h_8^{(\text{unshared})}\|_2]}$ & $\mathbf{\mathbb{E}[\|h_8^{(\text{unshared})}\|_2]}$ & $\mathbf{\mathbb{E}\left[\frac{\|h_8^{(\text{shared})}-h_8^{(\text{unshared})}\|_2}{\|h_8^{(\text{unshared})}\|_2}\right]}$ \\
        \midrule
        140M & 0.4588 & 56.3 & 0.82 \\
        370M & 2.2766 & 135.2 & 1.69 \\
        770M & 186.35 & 153.4 & 134.8 \\
        1.3B & 342.65 & 301.6 & 177.5 \\
        \bottomrule
    \end{tabular}
    \caption{Mean difference between shared and unshared final latents, mean unshared latent norm, and mean relative difference.}
    \label{tab:fixed-point-difference}
\end{table}

\paragraph{KV-sharing Fails to Converge Correctly.} As visualized by \cref{fig:fixed-point-failure}, we see that models that do not have KV-shareability from \cref{tab:kv-sharing} (e.g., Parcae 770M and 1.3B) seem to have disjoint sets between unshared and shared fixed points, while models with KV-shareability have approximately similar distributions. We capture this numerically in \cref{tab:fixed-point-difference}, where we find that models that have KV-shareability approach a relatively similar fixed point. For example, Parcae 140M and 370M have a relative difference of 0.82\% and 1.69\%, respectively, which means they approached the same relative fixed point. However, Parcae 770M and 1.3B have an average relative difference of 134.8\% and 177.5\%, respectively, which means that they converge to different points with and without KV-sharing. Thus, Parcae 770M and 1.3B fail to have KV-shareability as they do not converge to the same fixed points as the base model.

\begin{table}[!ht]
    \centering
    \setlength{\tabcolsep}{5pt}
    \renewcommand{\arraystretch}{1.15}
    \small
    \begin{tabular}{l|cc}
        \toprule
        \textbf{Model} & \textbf{Baseline: Mean Relative Difference} & \textbf{\sysname: Mean Relative Difference} \\
        \midrule
        140M & 0.0526 & 0.0059 \\
        370M & 0.1304 & 0.0059 \\
        770M & 0.0957 & 0.0070 \\
        1.3B & 0.1253 & 0.0049 \\
        \bottomrule
    \end{tabular}
    \caption{\textbf{Mean Relative Difference of Fixed Points With and Without KV-sharing Before and After Post-training with \sysname.} Mean relative difference at each model scale for the baseline and \sysname.}
    \label{tab:relative-difference}
\end{table}

We further verify these results empirically with models trained in \cref{sec:quality}, where we find that all trained baseline Parcae models fail to converge to the same fixed point under KV-sharing (see \cref{tab:relative-difference}), preventing them from having KV-shareability. However, after finetuning with \sysname, we empirically observe that models converge to the correct fixed point, as seen by \cref{tab:relative-difference}.

\section{Extended Finetuning Setup}
\label{sec:finetuning-extended}

In this section, we provide more details about the \sysname finetuning process.

\paragraph{Setup.} Following the training procedures of Parcae in \citet{prairie2026parcaescalinglawsstable}, all models are trained using batch sizes of 256 at context length of 2048. Recurrence depth is sampled stochastically from $\Lambda \sim \mathrm{Poisson}(\mu_\mathrm{rec})$, where $\mu_\mathrm{rec}=8$, as in pretraining. 

We use \sysname on two generations of models: previously released Parcae models (770m and 1.3B) \cite{prairie2026parcaescalinglawsstable} and new Parcae style models trained from scratch (including the MLA and KDA models in \cref{sec:quality}. The publicly released models are finetuned using \sysname with the same hyperparameters as their pretraining setup (see \citet{prairie2026parcaescalinglawsstable}) and use 140M \texttt{fineweb-edu} \cite{penedo2024finewebdatasetsdecantingweb} tokens. The new Parcae models are pretrained using hyperparameters described in \cref{sec:hyperparameters} on \texttt{climbmix} \cite{diao2025nemotronclimbclusteringbasediterativedata} tokens, and are further finetuned on 140M \texttt{climbmix} tokens using \sysname. Per-model finetuning learning rates and $\gamma$ values are listed in \cref{sec:hyperparameters}.

All \sysname results are averaged across three differently seeded finetuning runs evaluated on Core and Core-Ext using three different seeds as well.

\section{LoRA Finetuning for KV-sharing}
\label{sec:lora-ft}

In this section, we show that KV-sharing can be achieved using a LoRA\cite{hu2022lora} finetuning version of \sysname, which we denote $\sysname^\dagger$.

Following our findings in \cref{sec:elucdiation} that the KV-sharing error subspace is low-rank (\cref{fig:low-rank-subspace}), we finetune released Parcae 770M and 1.3B with low-rank adapters on the key and value weights of the first layer of the recurrent block. We use rank-8 adapters on both key and value weights and freeze the rest of the model weights, leading to 41K trainable parameters in the 770M model and 49K trainable parameters in the 1.3B model. We finetune using 65M \texttt{fineweb-edu} tokens with batch size 32 and sequence length 2048.

We use the same \sysname loss function defined in \cref{sec:methodology}. Hyperparameter details are in \cref{sec:hyperparameters}. In \cref{tab:quality-ft-lora}, both \sysname and $\sysname^\dagger$ enable KV-sharing while maintaining quality across benchmarks.

\begin{table}[!t]
    \centering
    \setlength{\tabcolsep}{4pt} %
    \renewcommand{\arraystretch}{1.1} %
    \small
   \begin{tabular}{cl|ccc|ccc|ccc}
        \toprule
        & \multirow{2}{*}{\textbf{Model}} & \multicolumn{3}{c}{\textbf{Val. PPL} ($\downarrow$)} & \multicolumn{3}{c}{\textbf{Core} ($\uparrow$)} & \multicolumn{3}{c}{\textbf{Core-Extended} ($\uparrow$)} \\
        \cmidrule(lr){3-5} \cmidrule(lr){6-8} \cmidrule(lr){9-11}
        & & Unshared & Shared & $\mathbf{\Delta}$ & Unshared & Shared & $\mathbf{\Delta}$  & Unshared & Shared & $\mathbf{\Delta}$  \\
        \midrule
        \multirow{3}{*}{\rotatebox[origin=c]{90}{\emph{770M}}}
& Parcae & 11.87 & 34.20 & +22.33 & 25.05 & 21.10 & -3.94 & \textbf{15.69} & 13.95 & -1.74 \\
& + Standard FT & 11.86 & 21.61 & +9.75 & 24.74 & 23.42 & -1.32 & 15.34 & 14.54 & -0.80 \\
& + \textbf{\sysname} & \textbf{11.85} & \textbf{11.86} & +0.01 & \textbf{25.15} & \textbf{25.19} & \textbf{+0.04} & 15.35 & \textbf{15.29} & \textbf{-0.06} \\
& + \textbf{\sysname}$^\dagger$ & 11.89 & 11.89 & \textbf{+0.00} & 24.69 & 24.71 & +0.02 & 15.32 & 15.19 & -0.13\\
            \cdashline{1-11}
        \multirow{3}{*}{\rotatebox[origin=c]{90}{\emph{1.3B}}}
& Parcae & \textbf{10.78} & 1.6e4 & +1.6e4 & 28.78 & 21.33 & -7.45 & 16.73 & 13.48 & -3.25 \\
& + Standard FT & 10.79 & 18.52 & +7.73 & \textbf{29.01} & 27.49 & -1.51 & 17.23 & 16.66 & -0.58 \\
& + \textbf{\sysname} & 10.79 & \textbf{10.78} & \textbf{-0.01} & 28.63 & \textbf{28.67} & +0.04 & \textbf{17.32} & \textbf{17.41} & \textbf{+0.09} \\
& + \textbf{\sysname}$^\dagger$  & 10.81 & 10.81 & 0.00 & 28.28 & 28.35 & \textbf{+0.07} & 16.94 & 16.96 & +0.02\\ 
        \bottomrule
    \end{tabular}
    \caption{\textbf{Quality of \sysname Finetuning Extended.} We provide the extended results of finetuning using \sysname and $\sysname^\dagger$. Both enable KV-sharing while maintaining quality. }
    \label{tab:quality-ft-lora}
    \vspace{-0.7em}
\end{table}

\section{Extended Discussion of Experimental Setup for MLA and KDA Evaluation}
\label{sec:extended-mla-kda}

In this section, we provide a more detailed description of the experimental setup for evaluating MLA \cite{deepseekai2024deepseekv2strongeconomicalefficient} and KDA \cite{kimiteam2025kimilinearexpressiveefficient} versions of Parcae \cite{prairie2026parcaescalinglawsstable}. We will fully discuss how \emph{KV-sharing} is done for both alternative sequence mixing methods. Hyperparameters and model configurations can be found in \cref{sec:hyperparameters} and \cref{sec:model-configurations}, respectively. \sysname finetuning is explained in \cref{sec:finetuning-extended}.

\paragraph{KV-sharing with Multi-Latent Attention} Similarly to MHA \cite{vaswani2023attentionneed}, \emph{KV-sharing} with MLA \cite{deepseekai2024deepseekv2strongeconomicalefficient} is done by discarding the latent cache of indeterminate recurrences, and keeping only the final recurrences latent cache. More formally, for an input $h_t \in \mathbb{R}^{n \times d}$ at recurrence $t$, the latent cache and the decoupled RoPE \cite{su2023roformerenhancedtransformerrotary} caches are calculated as
\begin{equation}
     c_t^{\text{KV}} = W^\text{KV} h_t \qquad \qquad k_t^{\text{RoPE}} = \text{RoPE} ( W^{\text{RoPE}} h_t ),
\end{equation}
where $\text{RoPE}(\cdot)$ denotes the RoPE rotation \cite{su2023roformerenhancedtransformerrotary}.
Thus with KV-sharing MLA models, we discard all $c_t^{\text{KV}},k_t^{\text{RoPE}}$ where $t < T$ keeping only the cache at $t = T$. 

\paragraph{KV-sharing with Kimi Delta Attention} Again, like both MHA \cite{vaswani2023attentionneed} and MLA \cite{deepseekai2024deepseekv2strongeconomicalefficient}, KV-sharing with KDA \cite{kimiteam2025kimilinearexpressiveefficient} is done by keeping only the final recurrences sequence-mixing hidden state. More formally, for an input $h_t$ at recurrence $t$, we focus on the linear update rule used in KDA \cite{kimiteam2025kimilinearexpressiveefficient}
\begin{equation}
    S_{t,i} = ( I - \beta_{t,i} k_{t,i} k_{t,i}^\top) \text{Diag}(\alpha_{t,i}) S_{t,i-1} + \beta_{t,i}  k_{t,i}  v_{t,i}^\top \in \mathbb{R}^{d_k \times d_v}
\end{equation}
where $\beta_i = W^\beta h_{t,i} \in \mathbb{R}$, $\alpha_i = W^\alpha h_{t,i} \in \mathbb{R}^{d_k}$, $k_i = W^{\text{K}} h_{t,i}\in \mathbb{R}^{d_k}$, and $v_i = W_v h_{t,i} \in \mathbb{R}^{d_v}$ and $W^\beta,W^\alpha,W^\text{K},W^{\text{V}}$ are learned projections. Thus with KV-sharing KDA models, we discard all $S_{t,i}$ where $t < T$ and all $i$ up to our current decode token.

\section{Extended Compression and Patching Results for Newly Pretrained Models}
\label{sec:compression-patching-pretrained}

In \cref{tab:compression-patching-mla-kda} we explore the KV-sharing ability of the base pretrained Parcae models introduced in \cref{sec:quality} (\cref{tab:quality}, see \cref{sec:model-configurations} for configurations and \cref{sec:hyperparameters} for hyperparameters). All four MHA models have KV-sharing errors. Preserving the first-layer KVs from the first recurrence in addition to normal KV-sharing maintains lossless quality for all models except the 1.3B model. The 1.3B model also has further quality degradation when using \emph{Naive Patching}. Interestingly, in \cref{tab:quality} using \sysname on the 1.3B model leads to near lossless quality, performing much better than the sub-optimal (in terms of memory efficiency) additional first layer, first recurrence KV-cache during KV-sharing.

We also explore the KV-sharing ability of MLA and KDA variants of Parcae in \cref{tab:compression-patching-mla-kda}. MLA variants of Parcae have the same failure mode of KV-sharing as released Parcae models. Mean patching fails to recover quality for our MLA model at 770M size. KDA models perform well on KV-sharing out-of-the-box.

\begin{table}[!t]
    \centering
    \setlength{\tabcolsep}{5pt}
    \renewcommand{\arraystretch}{1.15}
    \small
    \begin{tabular}{cl|cccc}
        \toprule
        & \multirow{2}{*}{\textbf{Model}}
        & \multicolumn{4}{c}{\textbf{Val. PPL} ($\downarrow$)} \\
        \cmidrule(lr){3-6}
        & & \emph{Unshared}
        & \emph{Layer $> 1$, Recurrence $> 1$}
      & \emph{Naive Patching}
        & \emph{Shared} \\
        \midrule
        \multirow{3}{*}{\rotatebox[origin=c]{90}{\emph{MHA}}}
        & 140M & 15.51 & 15.50 & 15.52 & 16.99 \\
        & 370M & 11.57 & 11.56 & 11.56 & 3723 \\
        & 770M & 9.78 & 9.80 & 10.49 & 36.66 \\
        & 1.3B & 8.58 & 8.92 & 12.77 & 9.42 \\
        \cdashline{1-6}
        \multirow{3}{*}{\rotatebox[origin=c]{90}{\emph{MLA}}}
        & 140M & 15.56 & 15.54 & 15.55 & 15.55 \\
        & 370M & 11.67 & 11.64 & 11.66 & 22.72 \\
        & 770M & 9.75 & 9.75 & 11.30 & 4894 \\
        \cdashline{1-6}
        \multirow{3}{*}{\rotatebox[origin=c]{90}{\emph{KDA}}}
        & 140M & 16.02 & 16.00 & 15.99 & 15.99 \\
        & 370M & 11.74 & 11.73 & 11.73 & 11.77 \\
        & 770M & 10.13 & 10.13 & 10.13 & 10.16 \\
        \bottomrule
    \end{tabular}
    \caption{\textbf{KV-Sharing Failures on New Models.} Preserving first-layer KVs from the first recurrence matches the unshared baseline across model sizes and sequence mixers. Mean patching fails to fully recover quality for MLA at 770M. KV-sharing fails for MLA models as size increases. Surprisingly, KDA models have KV-sharing out-of-the-box.}
    \label{tab:compression-patching-mla-kda}
\end{table}

\newpage
\section{Optimal Hyperparameter Transferability Across Recurrence}
\label{sec:hyperparameter-transfer}

While several works have shown how optimal hyperparameter transfer varies with tokens \citep{bjorck2025scalingoptimallrtoken,shuai2024scalinglawlanguagemodels} and model size \citep{yangTensorProgramsTuning2022}, we explore how scaling recurrence \citep{prairie2026parcaescalinglawsstable} is affected. Specifically, in this section, we investigate how optimal hyperparameters (e.g., learning rate, batch size, etc.) are affected by increasing recurrence at multiple parameter classes.

\paragraph{Setup.} We follow the model configuration and architecture backbone of \citet{prairie2026parcaescalinglawsstable} (see \cref{sec:model-configurations} for extend details), and train using \climbmix. We evaluate the transferability of learning rate, momentum, warm-up ratios, and batch size across $\meanrecurrence \in \{2,4,8\}$ using 140M, 370M, and 770M parameter backbones. To ensure robustness, for each hyperparameter, recurrence, backbones data point, we train three models using different seeds, $42, 1234, 2345$, and then use the average of these runs. This totals to 504 runs for this ablation and approximately 4,685 H100 hours.

\begin{figure}[t]
    \centering
    \includegraphics[width=\linewidth]{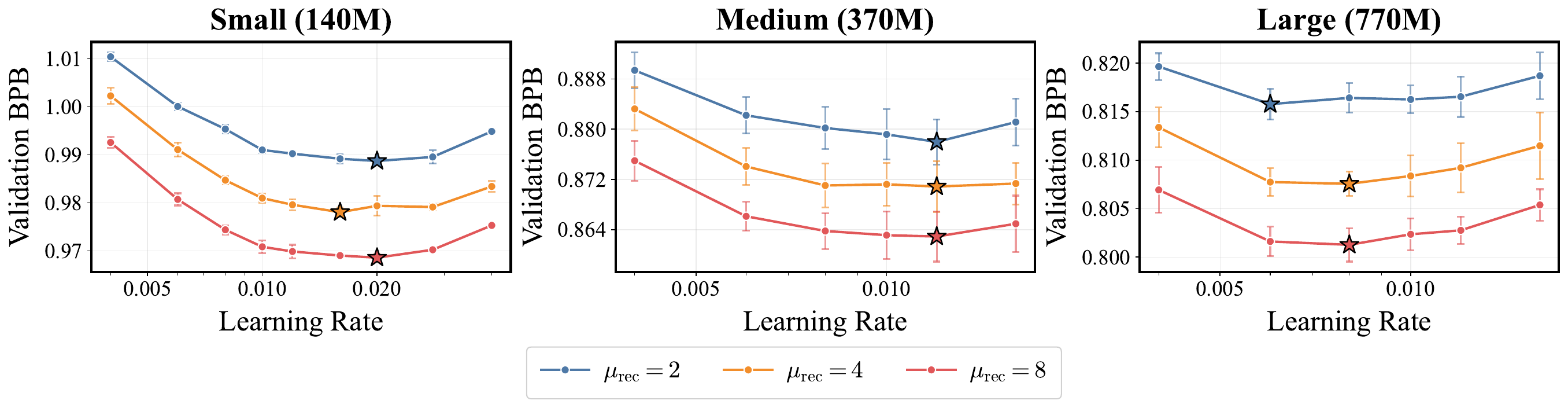}
    \caption{\textbf{Transferability of Learning Rate.} Optimal learning rate transfers strongly across deeper recurrences for all model backbones.}
    \label{fig:lr-transferability}
\end{figure}

\begin{figure}[t]
    \centering
    \includegraphics[width=\linewidth]{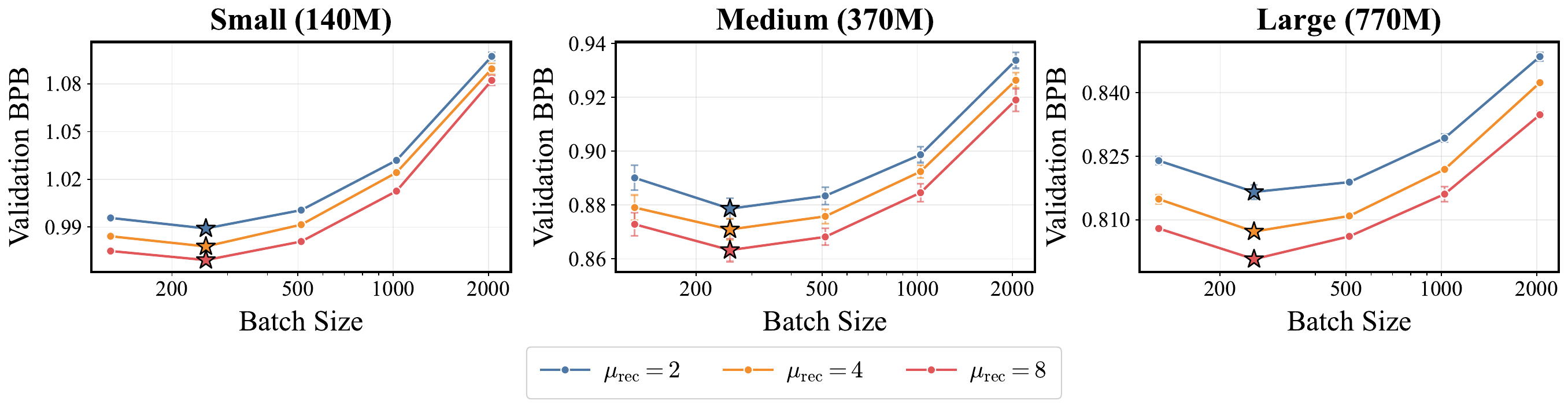}
    \caption{\textbf{Transferability of Batch Size.} Optimal batch size transfers strongly across deeper recurrences for all model backbones.}
    \label{fig:bs-transferability}
\end{figure}

\paragraph{Optimal Hyperparameters Transfer Strongly Across Increased Recurrence.} As seen in \cref{fig:lr-transferability} and \cref{fig:bs-transferability}, we find that across all model classes optimal learning rate and batch size transfer strongly. We further find this to be true of other minor hyperparameters (e.g., muon momentum and warmup ratio), with some minor variation in flat regions. This finding implies that lower recurrence (i.e., $\meanrecurrence = 2$) can be used to cheaply sweep for optimal hyperparameters, which in general remain optimal at higher recurrences. drastically reduces the necessary FLOPs required to find optimal hyperparameters for looped architectures.

Furthermore, as seen in \cref{tab:hyp-trans}, we observe a high Spearman rank correlation between the validation BPB across hyperparameter sweeps at different recurrences. While Spearman correlation is not 1 for all hyperparameters, we find that this is derived from noisy and flat hyperparameter landscapes. Thus, generally, we observe that most hyperparameters appear to transfer across recurrences, motivating our use of extremely low recurrence for hyperparameter tuning.

\begin{table}[t]
    \centering
    \vspace{-1em}
\begin{tabular}{ll|cccc|c}
\toprule
Hyperparameter & Model & 2 $\iff$ 4 & 2 $\iff$ 8 & 4 $\iff$ 8 & Mean & $\Delta$BPB ($\times 10^{-3}$) \\
\midrule
Learning Rate & 140M & 0.950 & 0.967 & 0.900 & 0.939 & 41.805 \\
Learning Rate & 370M & 0.943 & 1.000 & 0.943 & 0.962 & 26.483 \\
Learning Rate & 770M & 0.829 & 0.829 & 1.000 & 0.886 & 18.373 \\
Muon Momentum & 140M & 0.900 & 0.700 & 0.900 & 0.833 & 37.269 \\
Muon Momentum & 370M & 1.000 & 1.000 & 1.000 & 1.000 & 25.083 \\
Warmup Fraction & 140M & 0.600 & 0.900 & 0.700 & 0.733 & 22.614 \\
Warmup Fraction & 370M & 0.700 & 0.700 & 1.000 & 0.800 & 19.451 \\
Batch Size & 140M & 1.000 & 1.000 & 1.000 & 1.000 & 128.046 \\
Batch Size & 370M & 1.000 & 1.000 & 1.000 & 1.000 & 70.462 \\
Batch Size & 770M & 1.000 & 1.000 & 1.000 & 1.000 & 47.777 \\
\bottomrule
\end{tabular}
    \caption{\textbf{Spearman Rank Correlation of Hyperparameter curves across Recurrence.} Here $x \iff y$ is the Spearman correlation between the validation BPB hyperparameter curves of a model trained with $\meanrecurrence = x$ and a model trained with $\meanrecurrence = y$. Furthermore, $\Delta$ BPB indicates the validation BPB span across all (recurrence schedule $\times$ hyperparameter values) points for that specific parameter class. While Spearman correlation is not $1$ in all transfers, we observe that low Spearman correlation coincides with a lower $\Delta$ BPB span, which resolves to slight noise on a relatively flat landscape for hyperparameter tuning.}
    \label{tab:hyp-trans}
\end{table}

\begin{figure}[t]
    \centering
    \vspace{-0.4em}
    \includegraphics[width=\linewidth]{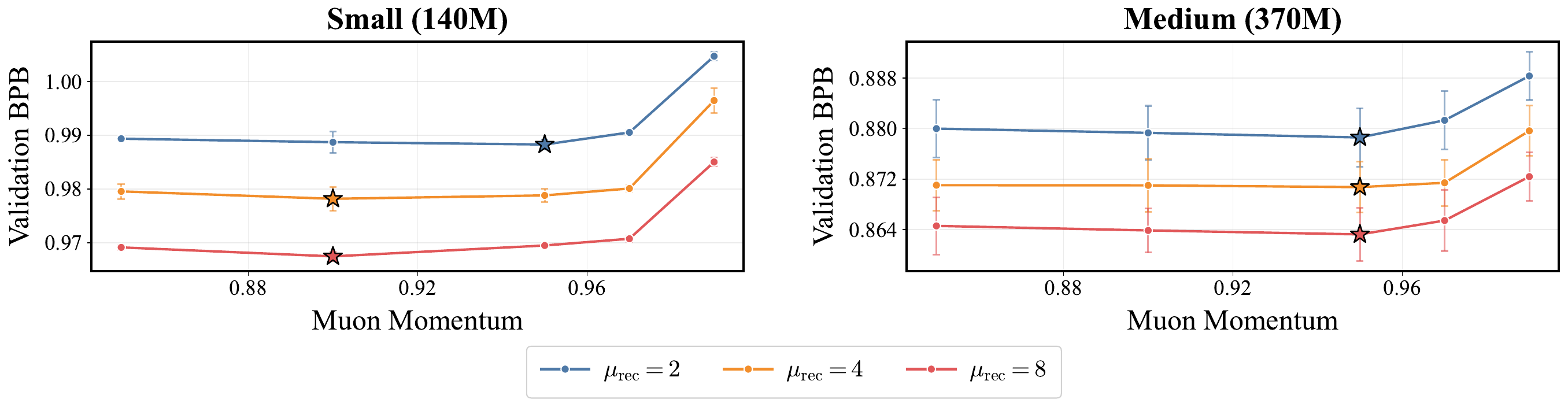}
    \vspace{-0.4em}
    \caption{\textbf{Transferability of Muon Momentum.} We observe optimal muon momentum to transfer well across recurrences, with some minor variation in relatively flat regions.}
    \vspace{-0.4em}
    \label{fig:muon-momentum-transfer}
\end{figure}

\begin{figure}[t]
    \centering
    \vspace{-0.4em}
    \includegraphics[width=\linewidth]{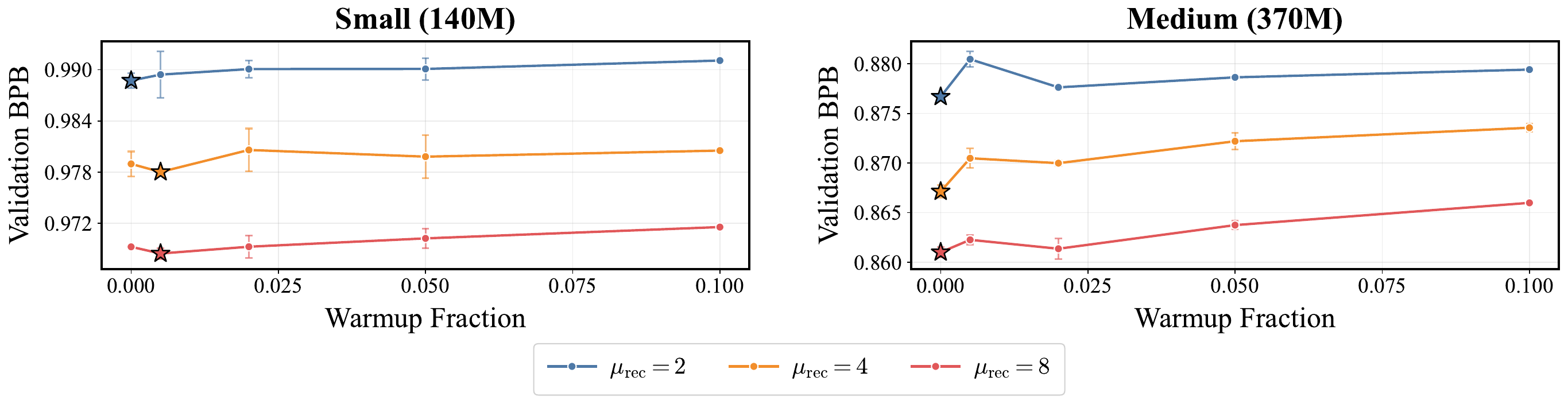}
    \vspace{-0.4em}
    \caption{\textbf{Transferability of Warmup Ratio.} We observe optimal warmup ratios to transfer well across recurrences, with some minor variation in relatively flat regions.}
    \vspace{-0.7em}
    \label{fig:warmup-transfer}
\end{figure}

\newpage
\section{Model Definitions and Configurations}
\label{sec:model-configurations}

In this section, we discuss the model definition used for our experiments. For our experiments done in \cref{sec:elucdiation} for our \emph{Stochastic} and \emph{Static} comparison, the models follow the configuration and hyperparameters presented in \cref{tab:model-definitions} and \cref{sec:hyperparameters} respectively, with two notable exceptions. First, for both models, we use a 1:7 param to token ratio. Second, for the \emph{Static} depth models, recurrence depth sampling does not follow the Poisson distribution, and is instead held deterministically at $\mu_\mathrm{rec}=8$ recurrences.

Our models present in pretraining results in \cref{sec:quality} also use the model configurations in \cref{tab:model-definitions} along with the hyperparameters present in \cref{tab:model-hyperparams}. Our optimizer parameter group assignment for all models (including those present in \cref{sec:elucdiation}) can be found in \cref{tab:adamw-v-muon}.

\begin{table}[!htbp]
  \centering
  \small
  \begin{tabular}{lccccc}
  \toprule
      & & Small (140M) & Medium (370M) & Large (770M) & XLarge (1.3B) \\
  \midrule
      \multirow{5}{*}{\rotatebox[origin=c]{90}{\textit{Architecture}}}
      & Layers in $\pre$ 
          & 2 & 4 & 6 & 8 \\
      & Layers in $\rec$ 
          & 2 & 4 & 6 & 8 \\
      & Layers in $\coda$ 
          & 2 & 4 & 6 & 8 \\
      & $d_{\text{model}}$
          & 768 & 1{,}024 & 1{,}280 & 1{,}536 \\
      & $d_{\text{intermediate}}$
          & 3{,}072 & 4{,}096 & 5{,}120 & 6{,}144 \\
  \midrule
     \multirow{3}{*}{\rotatebox[origin=c]{90}{\textit{MHA}}}
      & Attention
          & \multicolumn{4}{c}{Causal Self-Attention \citep{vaswani2023attentionneed} w/ QK-Norm  \citep{henry2020querykeynormalizationtransformers} } \\
      & Attention Heads
          & 6 & 8 & 10 & 12 \\
      & Head Dimension
          & \multicolumn{4}{c}{128} \\
      & Pos.\ Embed.
          & \multicolumn{4}{c}{RoPE ($\theta{=}50{,}000$) \citep{su2023roformerenhancedtransformerrotary}} \\
  \midrule
     \multirow{6}{*}{\rotatebox[origin=c]{90}{\textit{MLA}}}
      & Attention
          & \multicolumn{4}{c}{Causal Multi-Head Latent Attention \citep{deepseekai2024deepseekv2strongeconomicalefficient}} \\
      & Attention Heads
          & 6 & 8 & 10 & 12 \\
      & KV LoRA Rank
          & \multicolumn{4}{c}{512} \\
      & QK NoPE Head Dimension
          &\multicolumn{4}{c}{128} \\
      & QK RoPE Head Dimension  
          & \multicolumn{4}{c}{64} \\
      & V Head Dimension
          & \multicolumn{4}{c}{128} \\
      & Pos.\ Embed.
          & \multicolumn{4}{c}{RoPE ($\theta{=}50{,}000$) \citep{su2023roformerenhancedtransformerrotary}} \\
  \midrule
     \multirow{5}{*}{\rotatebox[origin=c]{90}{\textit{KDA}}}
      & Attention
          & \multicolumn{4}{c}{Causal Kimi Delta Attention \citep{kimiteam2025kimilinearexpressiveefficient}} \\
      & Attention Heads
          & 6 & 8 & 10 & 12 \\
      & Decay Gate     
          & \multicolumn{4}{c}{Lower bound at $-5$}  \\
      & Head Dimension
          & \multicolumn{4}{c}{128} \\
      & Convolution Kernel Size
          & \multicolumn{4}{c}{4} \\
  \midrule
      \multirow{8}{*}{\rotatebox[origin=c]{90}{\textit{Shared Details}}}
      & MLP
          & \multicolumn{4}{c}{$\text{ReLU}^2$ \citep{zhang2024relu2winsdiscoveringefficient}} \\
      & Value Embeddings
          & \multicolumn{4}{c}{Gated, alternating layers \citep{tian2023resformerscalingvitsmultiresolution}} \\
      & Vocab Size
          & \multicolumn{4}{c}{32{,}768} \\
      & Norm
          & \multicolumn{4}{c}{RMS-Norm (Pre-Norm) \citep{zhang2019rootmeansquarelayer}} \\
      & Context Length
          & \multicolumn{4}{c}{2{,}048} \\
      & Bias
          & \multicolumn{4}{c}{None} \\
      & Init.
          & \multicolumn{4}{c}{Scaled-zero \citep{takase2025spikemorestabilizingpretraining, nanochat}} \\
      & Tied Embeddings
          & \multicolumn{4}{c}{Yes} \\
  \midrule
      \multirow{5}{*}{\rotatebox[origin=c]{90}{\textit{Parcae}}}
      & Injection
          & \multicolumn{4}{c}{Diagonal} \\
      & State Init.
          & \multicolumn{4}{c}{\texttt{like-init} \citep{geiping_scaling_2025}} \\
      & $\meanrecurrence$
          & \multicolumn{4}{c}{8} \\
      & Backprop Depth
          & \multicolumn{4}{c}{4} \\
      & Sampling
          & \multicolumn{4}{c}{Poisson (truncated, per-sequence)} \\
  \bottomrule
  \end{tabular}
  \caption{Model definitions of baseline Parcae \cite{prairie2026parcaescalinglawsstable}, Parcae with MLA, and Parcae with KDA.}
  \label{tab:model-definitions}
\end{table}

\begin{table}[!htbp]
  \centering
  \small
  \begin{tabular}{ll}
  \toprule
      \textbf{Optimizer} & \textbf{Parameters} \\
  \midrule
      \multirow{7}{*}{AdamW \citep{kingma2017adammethodstochasticoptimization}}
      & Token embeddings (\texttt{wte}) \\
      & LM head (\texttt{lm\_head}) \\
      & Normalization layers (\texttt{RMSNorm}) \\
      & Value embedding gates (\texttt{ve\_gate}) \\
      & All 1D parameters \\
      & Injection parameters ($\A$, $\dt$, $\B$) \\
      & Readout projection ($\C$) \\
  \midrule
      \multirow{2}{*}{Muon \citep{jordan2024muon}}
      & Attention projections ($W_Q$, $W_K$, $W_V$, $W_O$) \\
      & MLP weights ($W_{\text{fc}}$, $W_{\text{proj}}$) \\
  \bottomrule
  \end{tabular}
  \caption{Optimizer parameter group assignment for Parcae.}
  \label{tab:adamw-v-muon}
\end{table}

\section{Hyperparameters and Training Details}
\label{sec:hyperparameters}

In this section, we will discuss the hyperparameter configuration used in \cref{sec:quality}.

\paragraph{Pretraining.} All pretraining hyperparameters are present in \cref{tab:model-hyperparams}, with the notable exception of the \emph{Stochastic} and \emph{Static} recurrence depth models present in \cref{sec:elucdiation}. Those exceptions are explained in \cref{sec:expanded-stochastic-fixed}.

\begin{table}[!htbp]
  \centering
  \small
  \begin{tabular}{lcccc}
  \toprule
      &  Small (140M) & Medium (370M) & Large (770M) & XLarge (1.3B) \\
  \midrule
       AdamW and Muon Learning Rate
          & $2e^{-2}$ & $1.2e^{-2}$ & $6e^{-3}$ & $5e^{-3}$ \\
       Warmup Ratio 
          & \multicolumn{4}{c}{0.05} \\
       Constant Decay Ratio 
          & \multicolumn{4}{c}{0.2} \\
       Gradient Clipping
          & \multicolumn{4}{c}{1.0} \\
    \midrule
       Muon Momentum
          & \multicolumn{4}{c}{0.95} \\
       Muon Weight Decay
          & \multicolumn{4}{c}{0.3}\\
       Muon Orthogonalization Steps 
          & \multicolumn{4}{c}{5} \\
    \midrule
       AdamW Betas
          & \multicolumn{4}{c}{(0.8,0.95)} \\
       AdamW Weight Decay
          & \multicolumn{4}{c}{0.1} \\
       AdamW Epsilon
          & \multicolumn{4}{c}{$1e^{-10}$} \\
  \midrule
       Batch Size
          & \multicolumn{4}{c}{256} \\
       Context Length
          & \multicolumn{4}{c}{2048} \\
       Param to Token Ratio
          & \multicolumn{4}{c}{1:80} \\
       Precision
          & \multicolumn{4}{c}{\texttt{bf16-mixed}} \\
  \bottomrule
  \end{tabular}
  \caption{Pretraining Hyperparameters for \texttt{climbmix} trained Parcae models (including MLA and KDA variants).}
  \label{tab:model-hyperparams}
\end{table}

\paragraph{\sysname and Standard Finetuning.} For both \sysname and Standard Finetuning, we sweep over learning rates, $\mathrm{lr} \in \{ 5e^{-6}, 1e^{-5}, 2e^{-5}, 3e^{-5}, \\ 5e^{-5}, 8^{-5}, 1.25e^{-4}, 2e^{-4}, 3e^{-4}, 4.5e^{-4}, 6e^{-4} \}$ and $\gamma \in \{0.01, 0.03, 0.1, 0.3, 1, 3 \}$. We perform these sweeps with an optimizer budget of 500 steps. For Standard Finetuning, we find 770M models (across MHA, MLA, and KDA) to perform best at a learning rate of $2e^{-4}$, while the 1.3B models perform best at $8e^{-5}$ on a held out validation set. For \sysname we find all models perform best at a learning rate of $1.25e^{-4}$ and with a $\gamma$ of $0.01$ on a held out validation set. Since we sweep over learning rates with an optimizer budget of 500 steps, we rescale our learning rate to account for the increase in optimizer budget (1800 steps). Specifically, all final learning rates are $\mathrm{lr_{optim}} * \frac{500}{1800}$.

\paragraph{$\sysname^\dagger$ Finetuning.} For $\sysname^\dagger$, we perform a more limited hyperparameter sweep on a limited model set: the released Parcae 770M and 1.3B models. We sweep over $\mathrm{lr} \in \{3e^{-4}, 1e^{-3} \}$ and $\gamma \in \{0.06, 0.12, 0.25, 1 \}$ for the full 1000 step optimizer budget. For both models, a learning rate of $1e^{-3}$ and a $\gamma$ of $0.06$ performed best on a held out validation set. We use a rank 8 adapter with an $\alpha=16$.

\newpage

\end{document}